\documentclass[letterpaper]{article}
\usepackage{aaai19}
\usepackage{times}
\usepackage{helvet}
\usepackage{courier}
\usepackage{url}
\usepackage{graphicx}
\usepackage[utf8]{inputenc} % allow utf-8 input
\usepackage[T1]{fontenc} % use 8-bit T1 fonts
\usepackage{booktabs} % professional-quality tables
\usepackage{amsfonts} % blackboard math symbols
\usepackage{nicefrac} % compact symbols for 1/2, etc.
\usepackage{microtype} % microtypography
\usepackage{amsmath}
\usepackage{amsthm}
\usepackage{commath}
\usepackage{color} % for colored text
\usepackage{subcaption}
\usepackage{bbm}
\usepackage{algorithm}
\usepackage[noend]{algpseudocode}
\usepackage{sidecap}
\graphicspath{ {images-arxiv/} }

\title{The Utility of Sparse Representations for Control in Reinforcement Learning}

\author{Vincent Liu\textsuperscript{1}, Raksha Kumaraswamy\textsuperscript{1}, Lei Le\textsuperscript{2}, Martha White\textsuperscript{1}\\
\textsuperscript{1}Department of Computing Science, University of Alberta, Edmonton, Canada\\
\{vliu1, kumarasw, whitem\}@ualberta.ca\\
\textsuperscript{2}Department of Computer Science, Indiana University Bloomington, Indiana, USA\\
leile@indiana.edu
}

\newcommand{\makevector}[1]{{\mathbf #1}}

\newcommand{\bvec}{{\makevector{b}}}

\newcommand{\wvec}{{\makevector{w}}}
\newcommand{\xvec}{{\makevector{x}}}

\newcommand{\weights}{{\makevector{w}}}

\newcommand{\thetavec}{{\boldsymbol{\theta}}}

\newcommand{\betahat}{\hat{\beta}}

\newcommand{\phivec}{{\boldsymbol{\phi}}}

\newcommand{\RR}{{\mathbb{R}}}

\newcommand{\calN}{{\mathcal{N}}}

\newcommand{\Wmat}{\mathbf{W}}

\newcommand{\Xmat}{\mathbf{X}}

\newcommand{\wdim}{d}

\newtheorem{corollary}{Corollary}

\newtheorem{theorem}{Theorem}

 \newcommand{\defeq}{:=}

\newcommand{\States}{\mathcal{S}}
\newcommand{\Actions}{\mathcal{A}}

\newcommand{\inv}{{\raisebox{.2ex}{$\scriptscriptstyle-1$}}}

\newcommand{\citep}[1]{\cite{#1}}

\begin{document}
\maketitle
\begin{abstract}
We investigate sparse representations for control in reinforcement learning. While these representations are widely used in computer vision, their prevalence in reinforcement learning is limited to sparse coding where extracting representations for new data can be computationally intensive.
Here, we begin by demonstrating that learning a control policy incrementally with a representation from a standard neural network fails in classic control domains, whereas learning with a representation obtained from a neural network that has sparsity properties enforced is effective.
We provide evidence that the reason for this is that the sparse representation provides locality, and so avoids catastrophic interference, and particularly keeps consistent, stable values for bootstrapping.
We then discuss how to learn such sparse representations. We explore the idea of Distributional Regularizers, where the activation of hidden nodes is encouraged to match a particular distribution that results in sparse activation across time. We identify a simple but effective way to obtain sparse representations, not afforded by previously proposed strategies, making it more practical for further investigation into sparse representations for reinforcement learning.
\end{abstract}

\section{Introduction}

Learning performance in artificial intelligence systems is highly dependent on the data representation---the features.
An effective representation captures important attributes of the state (or instance), as well as simplifies the estimation of predictors. Consider a reinforcement learning agent. A local representation enables the agent to more feasibly make accurate predictions for that local region, because the local dynamics are likely to be a simpler function than learning global dynamics. Additionally, such a representation can help prevent forgetting or interference \cite{mccloskey1989catastrophic,french1991using}, by only updating local weights, as opposed to dense representations where any update would modify many weights. At the same time, it is important to have a distributed representation \cite{bengio2009learning,bengio2013representation}, where the representation for an input is distributed across multiple features or attributes, promoting generalization and a more compact representation.

Such properties can be well captured by sparse representations: those for which only a few features are active for a given input (Figure \ref{fg:sparse_rep}). Enforcing sparsity promotes identifying key attributes, because it encourages the input to be well-described by a small subset of attributes. Sparsity, then, promotes locality, because local inputs are likely to share similar attributes (similar activation patterns) with less overlap to non-local inputs. In fact, many hand-crafted features are sparse representations, including tile coding \cite{sutton1996generalization,sutton1998reinforcement}, radial basis functions and sparse distributed memory \cite{kanerva1988sparse,ratitch2004sparse}.
Other useful properties of sparse representations---which can be seen as projecting data into a higher-dimensional space---include invariance \cite{goodfellow2009measuring,rifai2011contractive}; decorrelated features per instance \cite{foldiak1990forming}; improved computational efficiency for updating weights in the predictor, as only weights corresponding to active features need to be updates; and enabling linear separability in the high-dimensional space \cite{cover1965geometrical}, which facilitates the learning of a simple linear predictor.
Further, such sparse, distributed representations have been observed in the brain \citep{olshausen1997sparse,quian2010measuring,ahmad2015properties}.

\begin{figure}[t]
 \centering
 \includegraphics[width=0.45\textwidth]{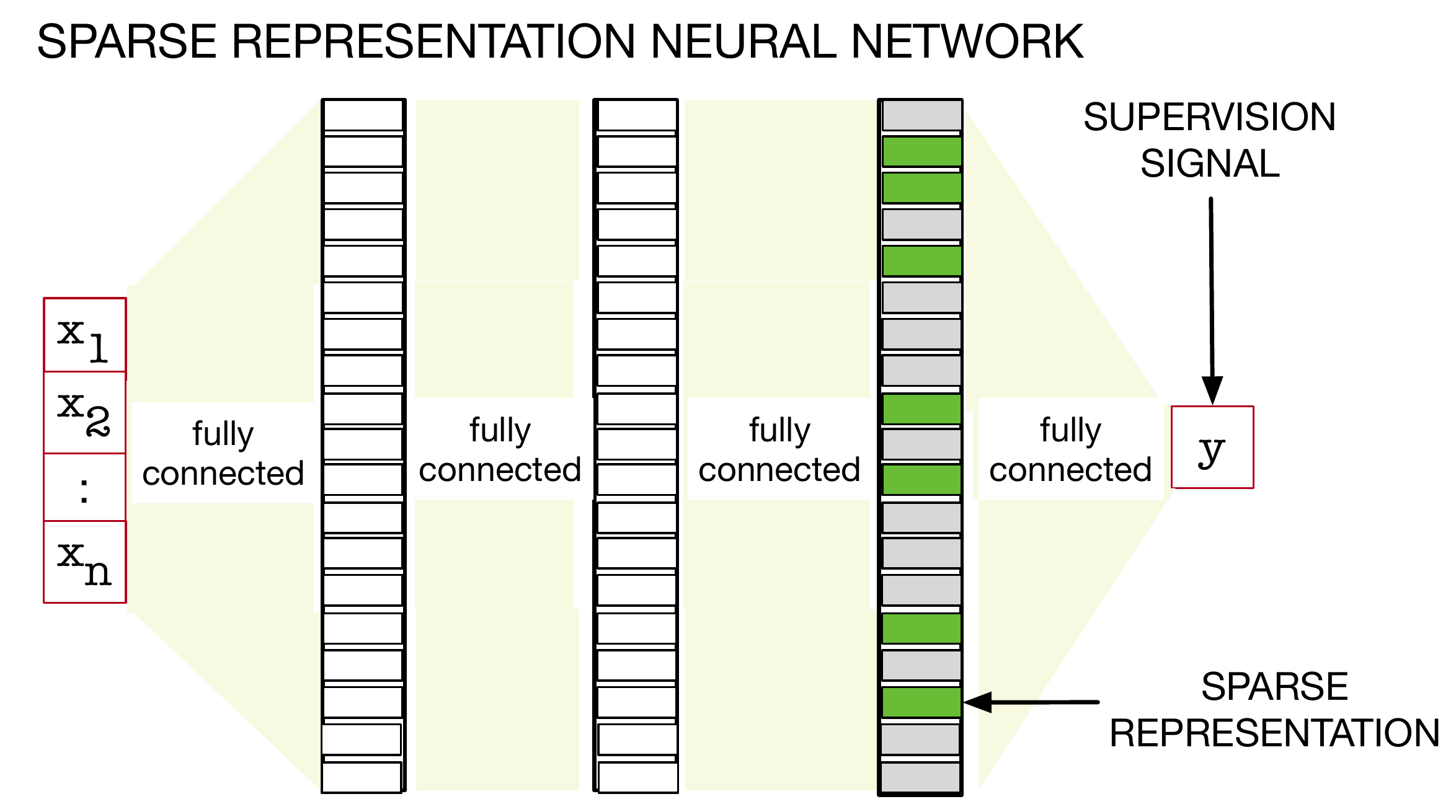}
 \caption{A neural network with dense connections producing a sparse representation: Sparse Representation Neural Network (SR-NN). The green squares indicate active (non-zero) units, making a sparse last hidden layer where only a small percentage of units are active. This contrasts a network with sparse connections---which is often also called sparse. Sparse connections remove
 connections
 %arrows
 between nodes,
 but are likely to still produce a dense representation.}
 \label{fg:sparse_rep}
\end{figure}

Traditionally, sparse representations have been common for control in reinforcement learning, such as tile coding and radial basis functions \citep{sutton1998reinforcement}. They are effective for incremental learning, but can be difficult to scale to high-dimensional inputs because they grow exponentially with input dimension. Neural networks much more feasibly enable scaling to high-dimensional inputs, such as images, but can be problematic when used with incremental training. Instead, techniques like target networks, inspired by batch methods such as fitted Q-iteration \cite{riedmiller2005neural}, have been necessary for many of the successes of control with neural networks. We provide some evidence in this paper that this modification is necessary with dense, but not sparse, networks because the reinforcement learning agent bootstraps off its own estimates. If the value in other states are overwritten, the agent will bootstrap off inaccurate estimate. Local representations, however, are much less likely to suffer from interference and these issues with bootstrapping. Learned sparse representations, then, are a promising strategy to obtain the benefits of previously common, fixed sparse representations with the scaling of neural networks.

Learning sparse representations, however, does remain a challenge. There have been some approaches developed to learning sparse representations incrementally, particularly through factorization approaches for dictionary learning \cite{mairal2009supervised,mairal2010online,le2017learning} or for general sparse distributions \cite{olshausen1997sparse,olshausen2002sparse,teh2003energy,ranzato2006efficient,ranzato2007sparse,lee2008sparse}, like Boltzmann machines. In sparse coding, for example, the sparse representation learning problem is formulated as a matrix factorization, where input instances are reconstructed using a sparse, or small subset, of a large dictionary. Many of the methods for general sparse distribution, however, are expensive or complex to train and those based on sparse coding have been found to have serious out-of-sample issues \cite{mairal2009supervised,lemme2012online,le2017learning}.

There are fewer methods using feedforward neural network architectures. Certain activation functions---such as linear threshold units (LTU) \cite{mcculloch1943logical} and rectified linear units (ReLU) \cite{glorot2011deep}---naturally provide some level of sparsity, but of course provide no such guarantees. Early work on catastrophic interference investigated some simple heuristics for encouraging sparsity, such as node sharpening \citep{french1991using}. Though catastrophic interference was reduced, the resulting networks were still quite dense.\footnote{There have been strategies developed for catastrophic interference that rely on rehearsal or dedicating subparts of the network to particular tasks. This work is a complementary direction for understanding catastrophic interference for a sequential multi-task setting. We explore specifically the utility of sparse representations for alleviating interference for RL agents learning incrementally on one task, but do not necessarily imply that it is the only strategy to alleviate such interference. The comparisons in this work, therefore, focus on other strategies to learn sparse representations.}
%k-sparse auto-encoders \cite{makhzani2013k} and Winnner-Take-All auto-encoders \cite{makhzani2015winner} both use a top-k criterion for sparsity. characterized by a hyperparameter \textit{k}.
k-sparse auto-encoders \cite{makhzani2013k} use a top-k constraint per instance: only the top $k$ nodes with largest activations are kept, and the rest are zeroed. Winnner-Take-All auto-encoders \cite{makhzani2015winner} use a k\% response constraint per node across instances, during training, to promote sparse activations of the node over time.
%k-sparse auto-encoders \cite{makhzani2013k} and Winnner-Take-All auto-encoders \cite{makhzani2015winner} use a simple heuristic to guarantee only $k$ units are active: for each input, only the top $k$ nodes with largest activations are kept, and the rest are zeroed.
These approaches, however, can be problematic---as we reaffirm in this work---because they tend to truncate non-negligible values or produce insufficiently sparse representations. Another line of work has investigated learning or specifying sparse activation functions for neural networks \cite{triesch2005agradient,ranzato2006efficient,lemme2012online,arpit2015regularized}, but used a sigmoid activation which is unlikely to result in sparse representations. They define sparsity based on norms of the vector, rather than activation level.

In this work, we first highlight that learned sparse representations can significantly improve control performance, under an incremental learning setting, compared to dense neural networks. We visualize the activation of the hidden nodes for the sparse representation as well as the action-values for particular states. These provide evidence that locality helps avoid catastrophic interference and improves accuracy of action-values for bootstrapping. We then investigate a simple strategy for encouraging sparsity in neural networks: Distributional Regularizers. This approach flexibly enables any desired architecture, simply with the addition of a KL divergence on the activation level for a node. We show that direct use of such a regularizer can cause dead filters or collapse---activation concentrating on a few nodes---potentially explaining why this simple strategy has not yet found wide-spread use. We show that a simple clipping is sufficient to obtain effective sparse representations, and conclude with a comparison to several other strategies for obtaining a sparse representation on the same benchmark domains.

\section{Background}
\label{Background}

In reinforcement learning (RL), an agent interacts with its environment, receiving observations and selecting actions to maximize a reward signal. The environment is formalized by a Markov decision process (MDP), with states $\States$, actions $\Actions$, transition probabilities $\text{Pr}:\States\times\Actions\times\States\to[0,1]$, rewards $R:\States\times\Actions\times\States\to\RR$ and discount function $\gamma:\States\times\Actions\times\States\to [0,1]$ \cite{white2017unifying}.

One algorithm for on-policy control is Sarsa, where the agent updates its action-values for its current policy and acts near-greedily according to these action-values. The action-values for a policy $\pi: \States \times \Actions \to [0,1]$ are the expected return for that policy, starting from state $s$ and action $a$:
\begin{align}
 Q^\pi(s, a) &= \mathbb{E}[G_t | S_t = s, A_t = a] \\
 &\text{ where, } G_t = R_{t+1} + \gamma_{t+1} G_{t+1} \nonumber
\end{align}

These action-values can be estimated with function approximation, such as with neural network. Because the expected return is a real-value target, such a neural network typically uses a linear activation on the last layer:
\begin{equation}
Q^\pi(s, a) \approx \hat{Q}_{\weights, \thetavec}(s, a) \defeq \phivec_\thetavec(s,a)^\top \weights
\label{eq:q_sa_1}
\end{equation}
where $\weights \in \RR^\wdim$ is the weights in the last layer and $\phivec_\theta : \States \times \Actions \to \RR^\wdim$ is the \emph{representation} learned by the network with weights $\thetavec$, composed of all the hidden layers in the network. The function $\phivec_\thetavec(s,a)$ corresponds to the last layer in the network, with $\thetavec$ the weights of the network. The efficacy of the action-value approximation, therefore, relies on this representation $\phivec_\thetavec(s,a)$.

\section{The Utility of Sparsity for Control}

We begin by highlighting the utility of sparsity for control before discussing how to learn sparse representations.
We show that two sparse representations---tile coding and sparse representation learned by a neural network (referred to as \textit{SR-NN} from hereon)---both significantly improve stability in control.
We choose tile coding, a static representation, as a baseline to compare to, as it known to perform very well in the benchmark RL domains we experiment with \cite{sutton1998reinforcement}.
We hypothesize that the main reason is due to catastrophic interference, which is much less problematic for the local representations typically provided by a sparse representations.
We show both that
SR-NN does
appear to have more stable action-values for bootstrapping, and that the learned sparse representation is local, providing some evidence for this hypothesis.

We evaluate control performance on four benchmark domains: Mountain Car, Puddle World, Acrobot and Catcher. All domains are episodic, with discount set to 1 until termination. We choose these domains because they are well-understood, and typically considered relatively simple. A priori, it would be expected that a standard action-value method, like Sarsa, with a two-layer neural network, should be capable of learning a near-optimal policy in all four of these domains. We provide details about the domains in the Appendix.

The experimental set-up is as follows. To extract a representation with a neural network,
to be used for control, we pre-train the neural network on a batch of data with a mean-squared temporal difference error (MSTDE) objective and
the applicable regularization strategies.
The training data consists of trajectories generated by a fixed policy that explores much of the space in the various domains.
For the SR-NN, we use our distributional regularization strategy, described in a later section.
This learned representation is then fixed, and used by a (fully incremental) Sarsa(0) agent for learning a control policy, where only the weights $\wvec$ on the last layer are updated. The meta-parameters for the batch-trained neural network producing the representation and the Sarsa agent were swept in a wide range, and chosen based on control performance. The aim is to provide the best opportunity for a regular feed-forward network (NN) to learn on these problems, as it is more sensitive to its meta-parameters than the SR-NN. Additional details on ranges and objectives are provided in the Appendix.

We choose this two-stage training regime to remove confounding factors in difficulties of training neural networks incrementally. Our goal here is to identify if a sparse representation can improve control performance, and if so, why. The networks are trained with an objective for learning values, on a large batch of data generated by a policy that covers the space; the learned representations are capable of representing the optimal policy. We investigate their utility for fully incremental learning. Outside of this carefully controlled experiment, we advocate for learning the representation incrementally, for the task faced by the agent.

\begin{figure}[h]
 \centering
 \includegraphics[width=0.98\linewidth]{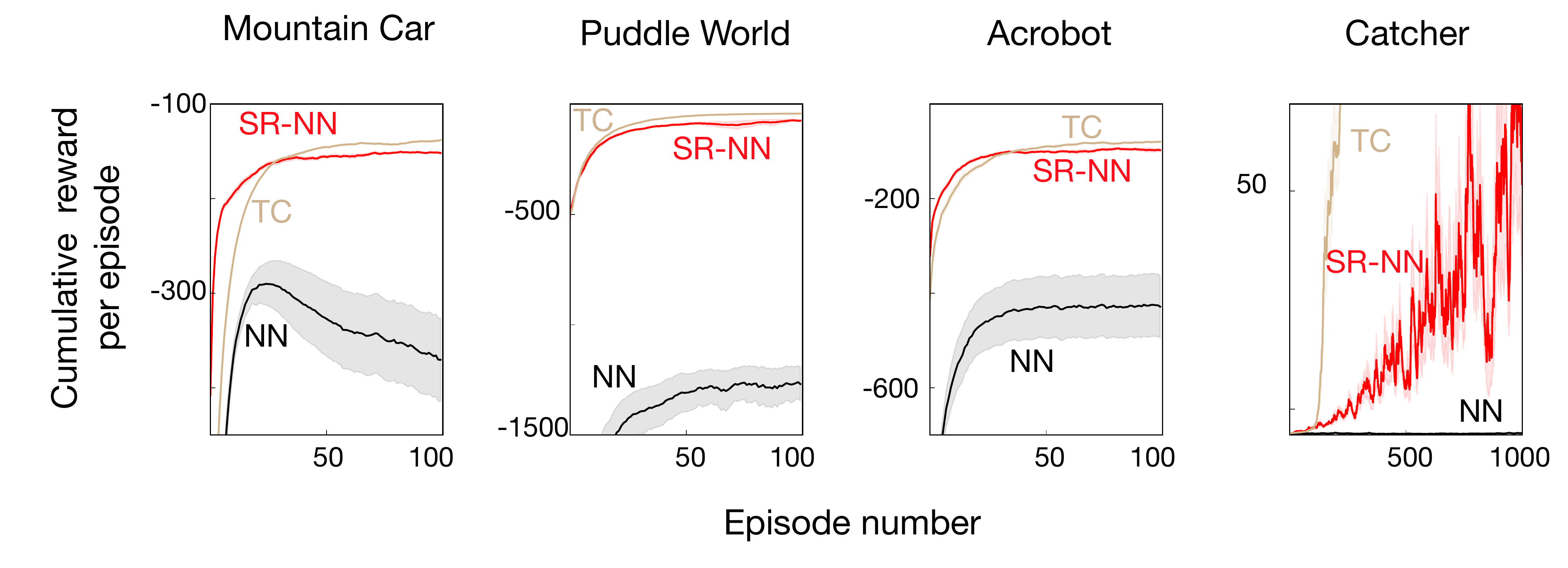}
 \caption{Learning curves for Sarsa(0) comparing SR-NN, Tile Coding and vanilla NN in the four domains.
 }
 \label{fg:exp-sparse}
\end{figure}

The learning curves for the four domains, with Tile-Coding (TC), SR-NN and NN
, are shown in Figure \ref{fg:exp-sparse}. Both SR-NN and NN used two-layers, of size $[32,256]$, with ReLU activations. The NNs performs surprisingly poorly, in some case increasing and then decreasing in performance (Mountain Car), and in others failing altogether (Catcher). In all the benchmark RL domains, the baseline sparse representation, TC, performs well, as expected. Specifically in Catcher, TC learns a close-to-optimal policy as the representation is powerful. The learned SR-NN performs as well in all domains, and is effective for learning in Catcher, whereas NN performs really poorly in all domains, and does not learn anything in Catcher. Both SR-NN and NN representations were trained in the same regime, with similar representational capabilities. Yet, the sparsity of SR-NN enables the Sarsa(0) agent to learn, where the regular feed-forward NN does not. We investigate this effect further in the next sets of experiments, to better understand the phenomenon.

\begin{figure}[t]
 \centering
 \includegraphics[width=0.98\linewidth]{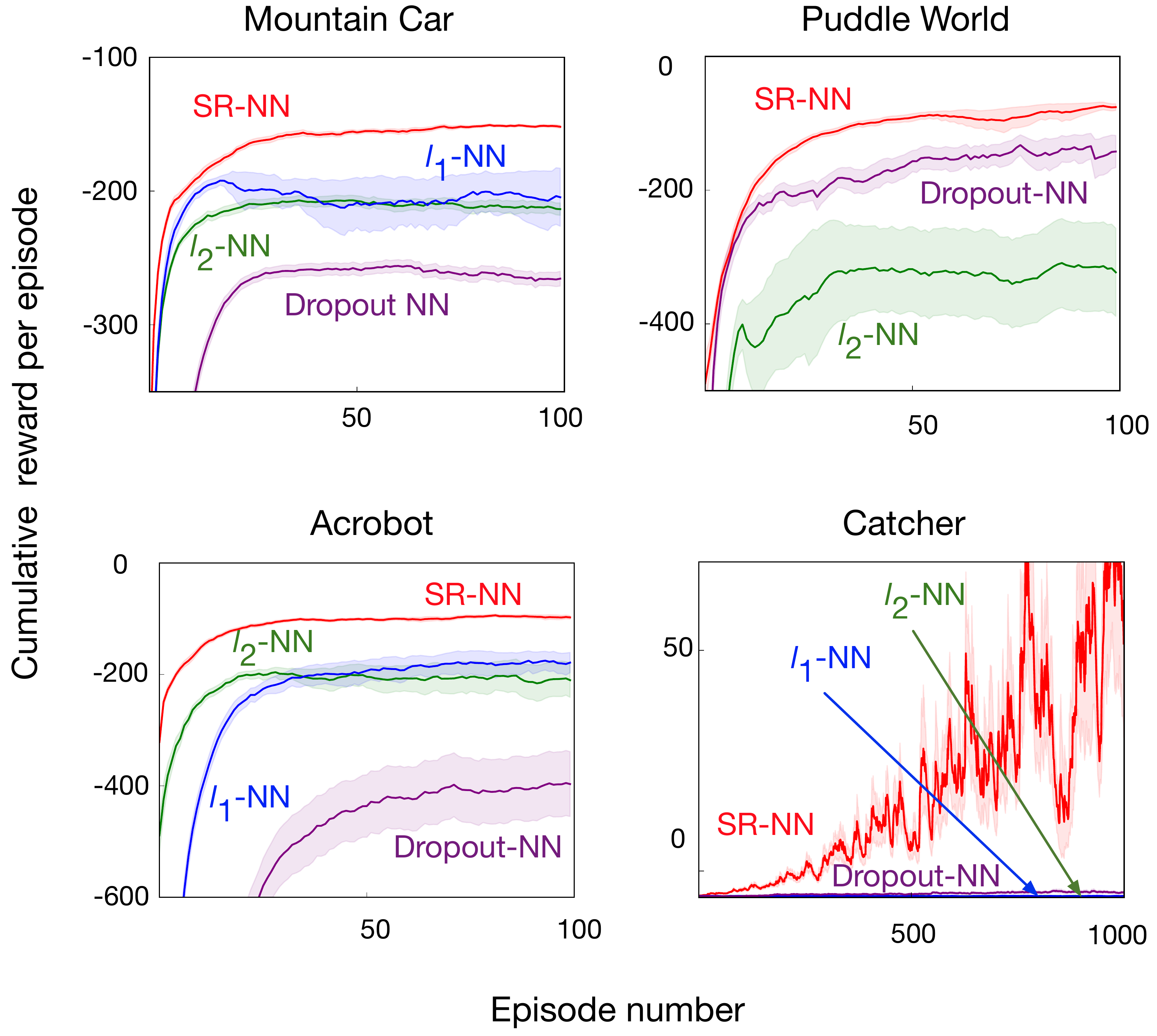}
 \caption{Learning curves for Sarsa(0) comparing SR-NN to the regularized representations.
 All representations except $\ell_1$-NN in Puddle World could reach goals more than 70 times out of 100. $\ell_1$ does poorly in Puddle World, and is not visible. }
 \label{fg:control-reg}
\end{figure}

\begin{figure*}[t]
 \includegraphics[width=0.98\linewidth,scale=0.2]{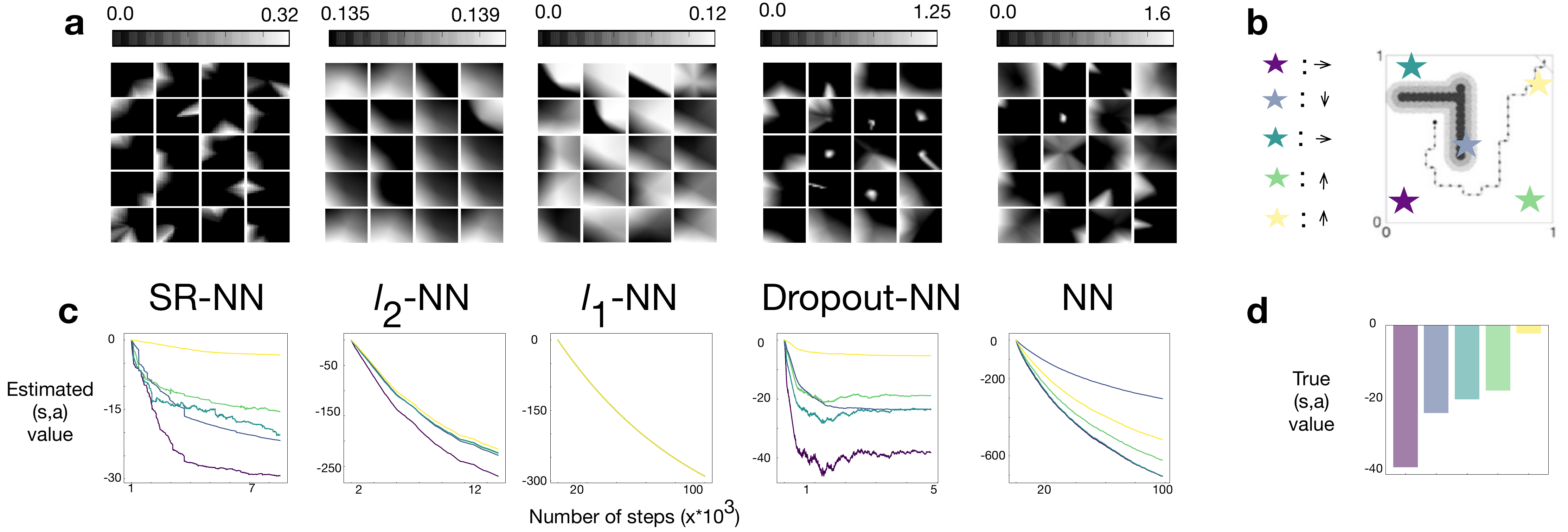}
 \caption{A study in Puddle World to investigate the effect of locality during on-policy control. (\textbf{a}) The activation maps for 20 randomly chosen neurons for different representations - each cell in the heatmap corresponds to the complete Puddle World state-space. The activation maps, and magnitude of activation of SR-NN are visibly sparser, and lower, when compared to Dropout and NN. $\ell_1$ and $\ell_2$ are quite dense in terms of activation area, whereas the magnitudes are really low - due to regularization of the network weights. (\textbf{b}) A visualization of the domain, denoting the selected state-action pairs used in the analysis. (\textbf{c}) The estimated state-action values for the selected configurations during on-policy control with Sarsa(0) ($\epsilon=0.1$), while utilizing the specific representation of interest. (Sarsa(0) budget - 100 episodes with 1k episode cut-off for all representations).
 (\textbf{d}) The true state-action values for the selected configurations with an $\epsilon=0.1$-optimal policy, estimated from 100k Monte Carlo rollouts.
 }
 \label{fg:controlProp_reg_pw}
\end{figure*}

To determine if the main impact of the sparse representation is simply from regularization, preventing overfitting, we tested several regularization strategies for the neural network. These include $\ell_2$ and $\ell_1$ on the weights of the network ($\ell_2$-NN and $\ell_1$-NN respectively) and Dropout on the activation \cite{srivastava2014dropout} (Dropout-NN). The $\ell_1$ regularizer actually encourages weights to go to zero, reducing the number of connections, but does not necessarily provide a sparse representation. In Figure \ref{fg:control-reg}, we can see that regularization is unlikely to account for the improvements in control. SR-NN performs well across all domains, whereas none of the regularization strategies consistently perform well. $\ell_1$-NN and $\ell_2$-NN perform well in Mountain Car during early learning, but fail in other domains. Dropout-NN performs poorly in all domains except Puddle World. Interestingly, in this one domain, Dropout-NN appears to have learned a sparse representation, based on the heatmap shown in Figure \ref{fg:controlProp_reg_pw}. It has been observed that Dropout can at times learn sparse representations \cite{banino2018vector}, but not consistently, as corroborated by our experiments.

We next investigate the hypothesis that locality is preventing catastrophic interference. We first investigate the locality of the representations, as well as examining the bootstrap values over time. We show results for Puddle World here, as it is an interpretable two-dimensional domain; similar experiments for other domains are in the Appendix.

Figure \ref{fg:controlProp_reg_pw}(a) shows the activation map of randomly selected hidden neurons with the different networks. We can see that each hidden neuron in SR-NN only responds to a local region of the input space, while some hidden neurons in NN respond to a large part of the space. Consequently, when one state is updated in a part of the space with the NN representation, it is more likely to significantly shift the values in other parts of the space, as compared to the more local SR-NN. The $\ell_2$-NN, and $\ell_1$-NN representations do not exhibit any discernible locality properties.
Dropout-NN does achieve some degree of locality in this domain, as mentioned earlier.

\begin{figure}[tb]
 \centering
 \includegraphics[width=\linewidth]{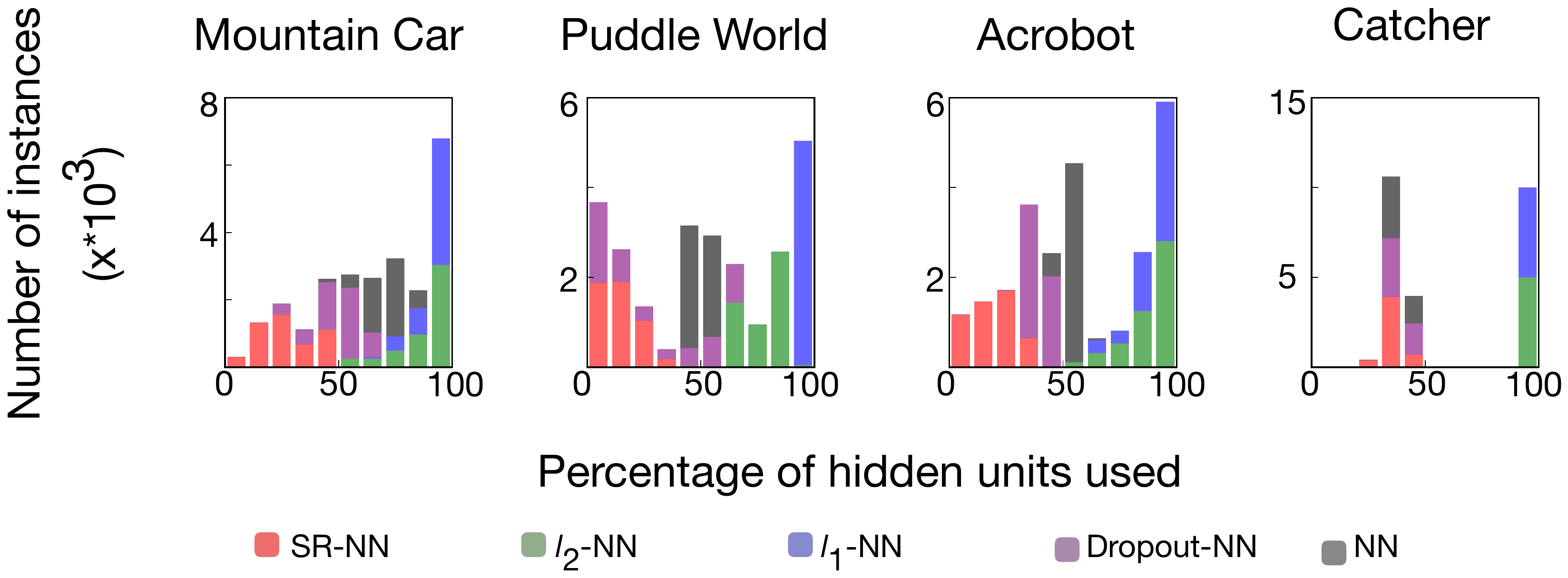}
 \caption{Instance sparsity comparing SR-NN to the regularized variants and vanilla NN.
 The percentage evaluation is designed to disregard units that are never active across all samples in the batch (dead units).}
 \label{fg:hist-reg}
\end{figure}

To show the stability (or lack of stability) of bootstrap
targets used during control, we select five states and evaluate their action-values for the optimal action over the course of learning. These states are distributed across the observation space, depicted in Figure \ref{fg:controlProp_reg_pw}(b).
The bootstrap estimates, that correspond to the algorithm settings for the learning curves, are plotted in Figure \ref{fg:controlProp_reg_pw}(c). We can see that the relative ordering of the value estimates is maintained with SR-NN and Dropout-NN, which were the two NNs effective for on-policy control, and that their values converge to near the true values (given in Figure \ref{fg:controlProp_reg_pw}(d)). The other representations, on the other hand, have very poor estimates. Moreover, these estimates seem to decrease together, suggesting interference is causing overgeneralization to reduce values in other states.

Finally, we report additional measures of locality, to determine if the successful methods are indeed sparse. The heatmaps provide some evidence of locality, but are more qualitative than quantitative. We report two qualitative measures: \emph{instance sparsity} and \textit{activation overlap}.
Instance sparsity corresponds to the percentage of active units for each input. A sparse representation should be instance sparse, where most inputs produce relatively low percentage activation. As shown in Figure \ref{fg:hist-reg}, SR-NN has consistently low instance sparsity across all four domains, with slightly higher level in Catcher, potentially explaining the noisy behaviour in that domain. Once again, Dropout-NN is noticeably more instance sparse on Puddle World, but less so on other domains. The NN representation, which has no regularization, has some instance sparsity, likely due to simply using ReLU activation. Interestingly, $\ell_1$-NN and $\ell_2$-NN actually produced less instance sparsity.

Activation overlap, introduced by \citeauthor{french1991using}\shortcite{french1991using}, reflects the amount of shared activation between any two inputs. We consider a variant of activation overlap that measures the number of shared activation between two representations, $\phivec(\xvec_1)$ and $\phivec(\xvec_2)$, for two samples, $\xvec_1$, and $\xvec_2$:
\begin{align*}
 \!\!\textrm{overlap}(\phivec(\xvec_1), &\phivec(\xvec_2)) \\
 &\!=\! \sum_j \!\mathbbm{1}[(\phivec_j(\xvec_1)>0) \land (\phivec_j(\xvec_2)>0)].
\end{align*}
We measure the activation overlap of the five chosen states, distributed across Puddle World.
If the overlap between two representations is zero, the interference would be zero. Updating
the value function with respect to one state, therefore, would not affect the other state's value.
Table \ref{tab:overlap-pw} shows the average overlap, and once again, a similar trend emerges where, SR-NN has significantly less overlap (about 8), with Dropout-NN showing the next least overlap (with about 30).

Overall, these results provide some evidence that (a) sparse representations can improve control performance in an incremental learning setting, (b) these sparse representations appear to provide locality and (c) this locality reduces interference and improves accuracy of bootstrap values in Sarsa(0).
These results are a first step, and warrant further investigation. They do nonetheless motivate that learning
%SR-NNs
sparse representations
could be a promising direction for control in reinforcement learning.
In the next section, we discuss how we actually obtain such sparse representations (SR-NN).

\section{Distributional Regularizers for Sparsity}
\label{sparsity}

In this section, we describe how to use Distributional Regularizers to learn sparse representations with neural networks.\footnote{The idea was originally introduced for neural networks with Sigmoid activations in an unpublished set of notes \cite{ng2011sparse}, and as yet has not been systematically explored. When used out-of-the-box, we found important limitations in the learned representations, including from using Sigmoid activations instead of ReLU and from using the KL to a specific distribution. We explore the idea in-depth here, to make it a practical option for learning sparse representations.} We introduce a Set Distributional Regularizer, which when paired with ReLU activations enables sparse representations to be learned, as we demonstrate in the next section. We first describe how to define Distributional Regularizers on neural networks, and then discuss the extension to a Set Distributional Regularizer, and motivation for doing so.

\begin{table}
\begin{center}
\begin{tabular}{ c | c | c | c | c }
 SR-NN & $\ell_2$-NN & $\ell_1$-NN & Dropout-NN & NN \\
 \hline
 \textbf{8.8} & 111.5 & 142.5 & 31.2 & 54.0 \\
\end{tabular}
\end{center}
 \caption{Activation overlap in Puddle World. The numbers are the average overlap over all pairs of selected states. For example, SR-NN has an average of 8.8 shared activation over all pais of 5 selected states defined in Figure \ref{fg:controlProp_reg_pw} (a). }
 \label{tab:overlap-pw}
\end{table}

The goal of using Distributional Regularizers is to encourage the distribution of each hidden node---across samples---to match a desired target distribution. In a neural network, we can view the hidden nodes, $Y_1, \ldots, Y_\wdim$, as random variables, with randomness due to random inputs. Each of these random variables $Y_j$ has a distribution $p_{\hat{\beta}_j(\theta)}$, where the parameters {\small$ \hat{\beta}_j(\theta)$} of this distribution are induced by the weights $\thetavec$ of the neural network:
\begin{equation*}
p_{\hat{\beta}_j(\theta)}(y) = \int_{s \in \States} p(s) p( \phivec_{j,\thetavec}(s) = y) d s
.
\end{equation*}
This provides a distribution over the values for the feature $\phivec_{j,\thetavec}(s)$, across inputs $s$.
A Distributional Regularizer is a KL divergence $KL(p_{\beta}||p_{\betahat_j(\thetavec)})$ that encourages this distribution to match a desired target distribution $p_\beta$ with parameter $\beta$.

Such a regularizer can be used to encourage sparsity, by selecting a target distribution that has high mass or density at zero.
Consider a Bernoulli distribution for activations, with $Y_j \in \{0,1\}$. Using a Bernoulli target distribution with $\beta = 0.1$, giving $p_\beta(Y = 1) = 0.1$, encodes a desired activation of 10\%. As another example, for continuous nonnegative $Y_j$, the target distribution can be set to an exponential distribution $p_\beta(y) = \beta^\inv \exp(-y/\beta)$, which has highest density at zero with expected value $\beta$. Setting $\beta = 0.1$ encourages the average activation to be $0.1$ and increases density on $y = 0$.

The efficacy of this regularizer, however, is tied to the parameterization of the network, which should match the target distribution. For a ReLU activation, for example, which has a range $[0,\infty)$, a Bernoulli target distribution is not appropriate. Rather, for the range $[0,\infty)$, an exponential distribution is more suitable. For a Sigmoid activation, giving values between $[0,1]$, a Bernoulli is reasonably appropriate. Additionally, the parametrization should be able to set activations to zero. The ReLU activation naturally enables zero values \citep{glorot2011deep}, by pushing activations to negative values. The addition of a Distributional Regularizer simply encourages this natural tendency, and is more likely to provide sparse representations. Activations under Sigmoid and tanh, on the other hand, are more difficult to encourage to zero, because they require highly negative input values or input values exactly equal to 0.5, respectively, to set the hidden node to zero.
For these reasons, we advocate for ReLU for the sparse layer, with an exponential target distribution.

Finally, we modify this regularizer to provide a Set Distributional Regularizer, which does not require an exact level of sparsity to be achieved. It can be difficult to choose a precise level of sparsity, making the Distributional Regularizer prone to misspecification. Rather, the actual goal is typically to obtain \emph{at least} some level of sparsity, where some nodes can be even more sparse. For this modification, we specify that the distribution should match any of a set of target distributions $Q_\beta$, giving a \emph{Set KL}: $\min_{p \in Q_\beta} KL(p || p_{\betahat_j(\thetavec)})$. Generally, this Set KL can be hard to evaluate. However, as we show below, it corresponds to a simple clipped KL-divergence for certain choices of $Q_\beta$, importantly including for exponential distributions where $Q_\beta = \{ p_{\tilde{\beta}} \ | \tilde{\beta} \le \beta\}$.

\begin{theorem}[Set KL as a Clipped-KL]
 Let $p_{\eta}$ be a one-dimensional exponential family distribution with the natural parameter $\eta$, $B = [\eta_1, \eta_2]$ be a convex set in the natural parameter space and $Q_B=\{p_{\eta}:\eta \in B\}$. Then the Set KL divergence
 \begin{equation}
 SKL(Q_B || p_\eta) \defeq \min_{p \in Q_B} KL(p || p_{\eta})
 \end{equation}
 is (a) non-negative (b) convex in $\eta$ and (c) corresponds to a simple clipped form
 \begin{equation}
 SKL(Q_B||p_\eta) = \left\{ \begin{array}{cc}KL(p_{\eta_2} || p_{\eta}) & \text{if } \eta > \eta_2\\
 KL(p_{\eta_1} || p_{\eta}) & \text{if } \eta < \eta_1\\
 0 & \text{ else }
 \end{array} \right.
 \end{equation}
 \label{thm:ckl}
\end{theorem}
\begin{proof}
 For exponential families, the KL divergence correspond to a Bregman divergence \cite{banerjee2005clustering}:
 \begin{equation*}
 KL(p_{\eta_1}||p_{\eta}) = D_{F}(\eta || \eta_1)
 \end{equation*}
 for a convex potential function $F$ that depends on the exponential family.
 Hence, we have
 \begin{equation*}
 SKL(Q_B ||p_{\eta}) = \arg\min_{\tilde{\eta} \in B} D_F(\eta || \tilde{\eta})
 \end{equation*}
 If $\eta \in B$, this minimum over Bregman divergences is clearly zero.
 If $\eta < \eta_1$ and $\eta > \eta_2$, we have to consider the minimization. The Bregman divergence
 is not necessarily convex in the second argument. Instead, we can rely on convexity of the set $B$.
 Taking the derivative of $D_F(\eta || \tilde{\eta})$ wrt $\tilde{\eta}$, we get
 \begin{align*}
 \frac{d}{d \tilde{\eta}} D_F(\eta || \tilde{\eta})
 &= \frac{d}{d \tilde{\eta}} \left[ F(\eta) - F(\tilde{\eta}) - (\eta - \tilde{\eta}) \frac{d}{d \tilde{\eta}} F(\tilde{\eta})\right]\\
 &= - \frac{d}{d \tilde{\eta}} F(\tilde{\eta}) + \frac{d}{d \tilde{\eta}} F(\tilde{\eta}) - (\eta - \tilde{\eta}) \frac{d^2}{d \tilde{\eta}^2} F(\tilde{\eta})\\
 &= -\frac{d^2}{d \tilde{\eta}^2} F(\tilde{\eta}) (\eta - \tilde{\eta})
 \end{align*}
 Now because $F$ is convex, $-\frac{d^2}{d \tilde{\eta}^2} F(\tilde{\eta})$ is always negative. The derivative, then, is negative when $\tilde{\eta} < \eta$, indicating $\tilde{\eta}$ should be increased to decrease $D_F(\eta || \tilde{\eta})$. Similarly, when $\tilde{\eta} > \eta$, the derivative is positive, indicating $\tilde{\eta}$ should be decreased to decrease $D_F(\eta || \tilde{\eta})$. This derivative, then, points $\tilde{\eta}$ to the boundaries when $\eta \notin B$, respectively to the boundary points closest to $\eta$.
\end{proof}
\begin{corollary}[SKL for Exponential Distributions]\label{cor_skl}
  %\raksha I think the natural parameter is $\eta = -\beta^\inv$. The argument still holds, except it's a little wierd because $\eta$ would need to be in the set $(-\infty, -\beta]$ (which kinda implicitly assumes $1.0/0 = \infty$). We don't have to talk about it though -- we can just fix the natural parameter definition below and we should still be good.
 For $p_\beta$ an exponential distribution, with natural parameter $\eta = - \beta^\inv$, and $B = (0, \beta]$, then
 \begin{equation}
 SKL(Q_B || p_{\betahat}) \!= \!\left\{ \begin{array}{cc} \log\betahat + \frac{\beta}{\betahat} - \log\beta - 1& \text{if } \betahat > \beta \\
 0 & \text{ else }
 \end{array} \right. \!\!\!\!\!\!
 \end{equation}
\end{corollary}
We use the SKL in Corollary \ref{cor_skl}, to encode a sparsity level of at least $\beta$---rather than exactly $\beta$---for the last layer in a two-layer neural network with ReLU activations.
This regularizer was used
to encourage sparse activations
for SR-NN in the preceding section.
We include pseudocode for optimizing the regularized objective with the SKL, in Algorithm \ref{alg:opt} in the Appendix. % \ref{alg:opt}

\section{Evaluation of Distributional Regularizers}

In this section, we investigate the efficacy of Distributional Regularizers for obtaining sparsity. There are a variety of possible choices with Distributional Regularizers, including activation function and corresponding target distribution and using a KL versus a Set KL.
In this section, we investigate some of these combinations, particularly focusing on the difference in sparsity and performance when using (a) KL versus SKL; (b) Sigmoid (with a Bernoulli target distribution) versus ReLU (with an Exponential target distribution); and
(c) previous strategies to obtain sparse representations
 versus the proposed variant of the Distributional Regularizer.

In the first set of experiments, we compare the instance sparsity of KL to Set KL, with ReLU activations and Exponential Distributions (ReLU+KL and ReLU+SKL).
Figure \ref{fg:hist-dist} shows the instance sparsity for these, and for the NN without regularization.
Interestingly, ReLU+KL actually reduces sparsity in several domains, because the optimization encouraging an exact level of sparsity is quite finicky. ReLU+SKL, on the other hand, significantly improves instance sparsity over
the NN.
This instance sparsity again translates into control performance, where ReLU+KL does noticeably worse than ReLU+SKL across the four domains in Figure \ref{fg:exp2-dist}. Despite the poor instance sparsity, ReLU+KL does actually seem to provide some useful regularity, that does allow some learning across all four domains. This contrasts the previous regularization strategies, $\ell_2$, $\ell_1$ and Dropout, which all failed to learn on at least one domain, particularly Catcher.

\begin{figure}[t]
 \centering
 \includegraphics[width=\linewidth]{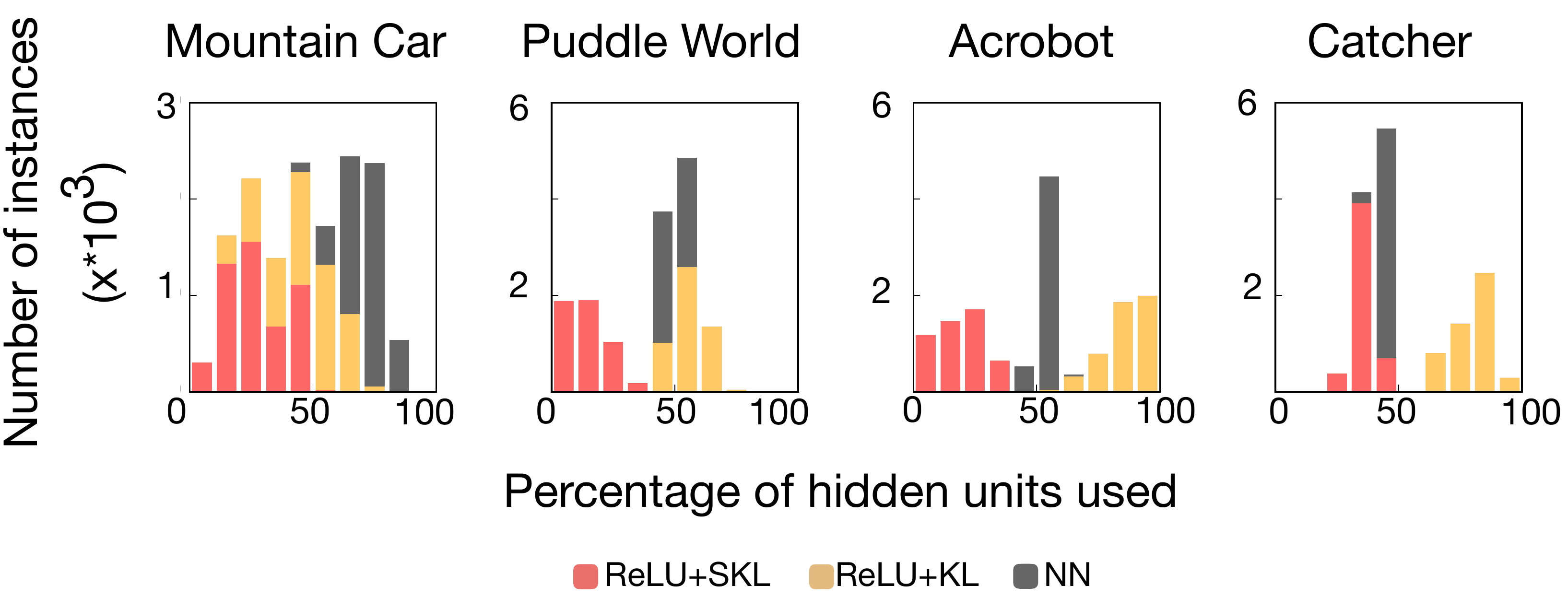}
 \caption{Instance sparsity as evaluated on a batch of test data comparing ReLU+KL and ReLU+SKL to NN. While ReLU+KL can make representations denser than just NN, ReLU+SKL always results in sparser representations.
 }
 \label{fg:hist-dist}
\end{figure}

In the next set of experiments, we compare Sigmoid (with a Bernoulli target distribution) versus ReLU (with an Exponential target distribution). We included both KL and Set KL, giving the combinations ReLU+KL, ReLU+SKL, SIG+KL, and SIG+SKL. We expect Sigmoid with Bernoulli to perform significantly worse---in terms of sparsity levels, locality and performance---because the Sigmoid activation makes it difficult to truly get sparse representations. This hypothesis is validated in the learning curves in Figure \ref{fg:exp2-dist} and the heatmaps for Puddle World in Figure \ref{fg:heatmap-dist}. SIG+KL and SIG+SKL perform poorly across domains, even in Puddle World, where they achieved their best performance. Unlike ReLU with Exponential, here the Set KL seems to provide little benefit. The heatmaps in Figure \ref{fg:heatmap-dist} show that both versions, SIG+KL and SIG+SKL,
cover large porions of the space, and do not have local activations for hidden nodes. In fact, SIG+KL and SIG+SKL use all the hidden nodes for all the samples across domains, resulting in no instance sparsity.

\begin{figure}[ht]
 \centering
 \includegraphics[width=0.9\linewidth,scale=0.5]{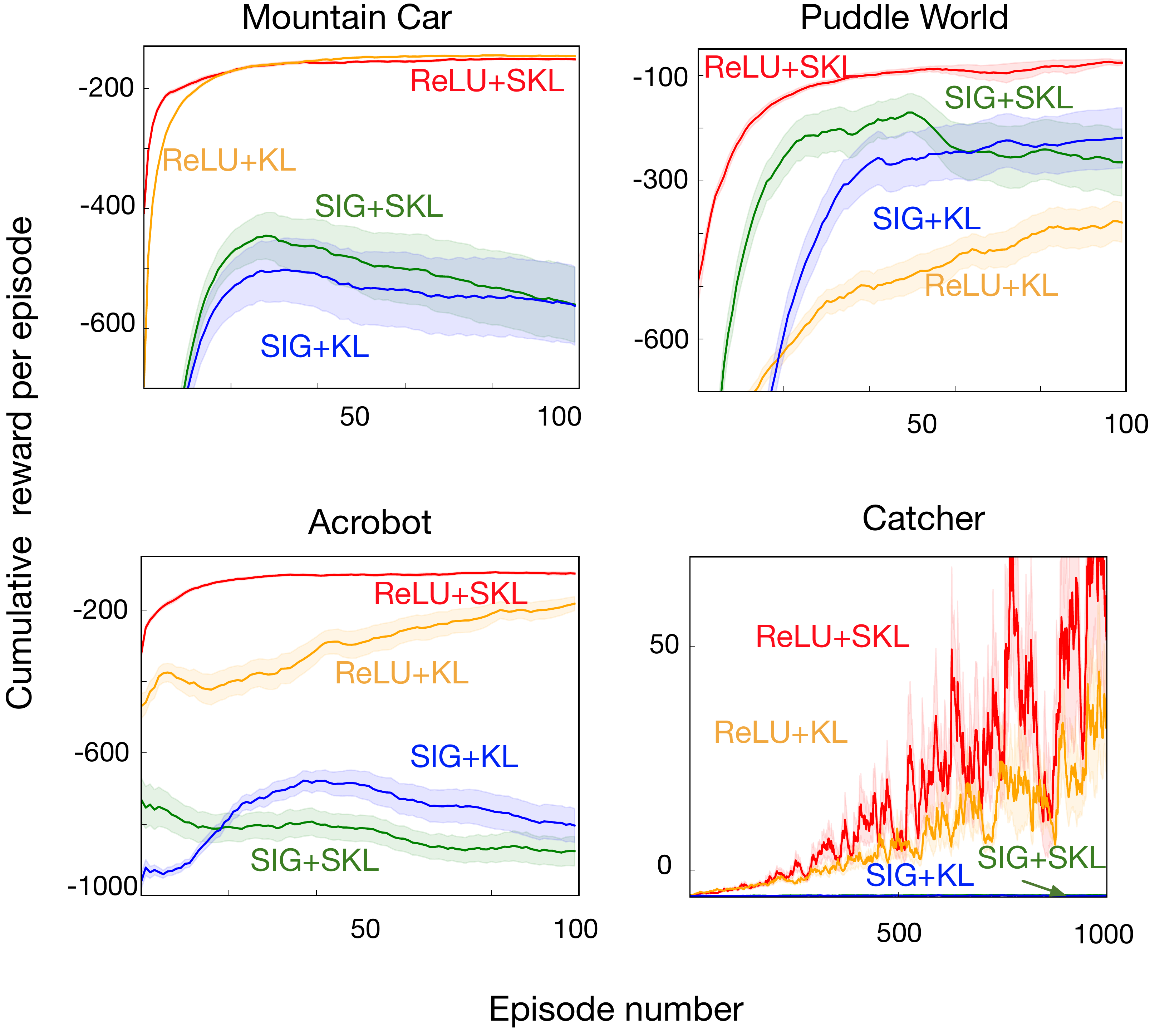}
 \caption{Learning curves for Sarsa(0) with different Distributional Regularizers.}
 \label{fg:exp2-dist}
\end{figure}

\begin{figure}[ht]
 \centering
 \includegraphics[width=0.9\linewidth]{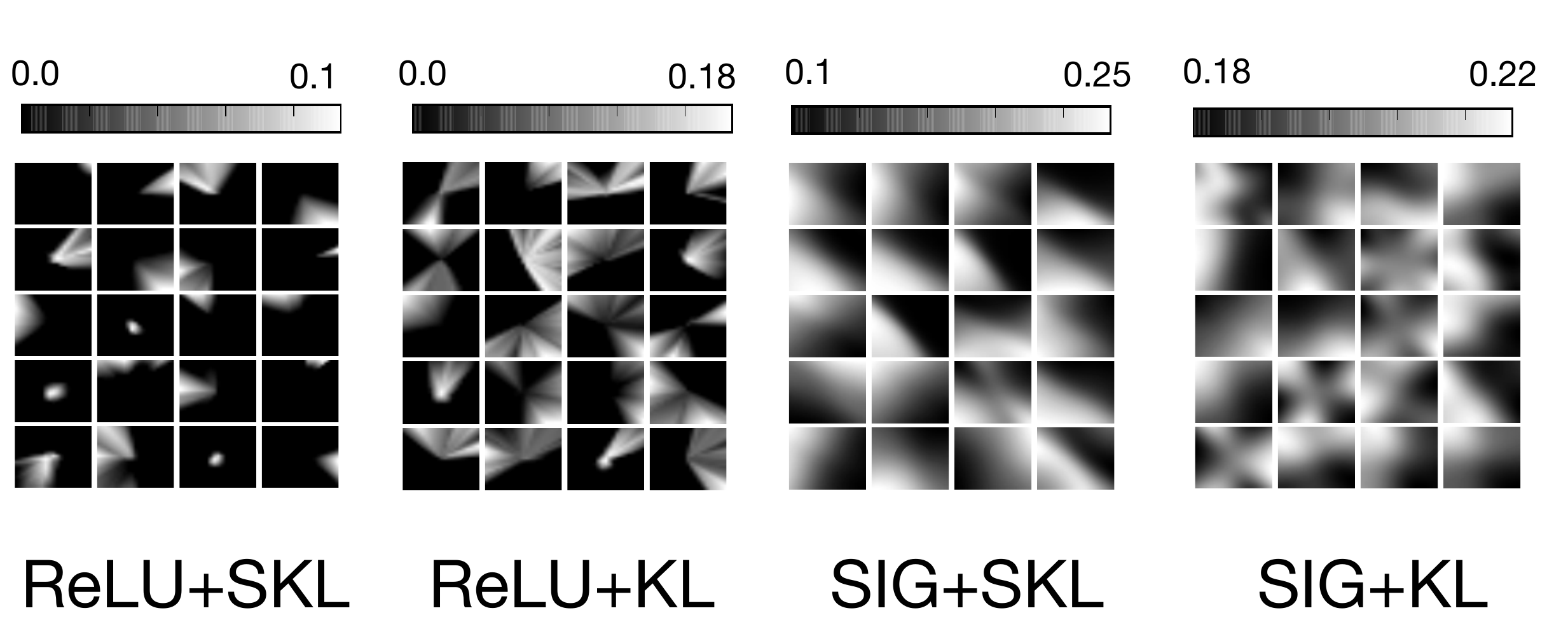}
 \caption{Heatmaps of activations with different Distributional Regularizers in Puddle World.}
 \label{fg:heatmap-dist}
\end{figure}

\begin{figure}[ht]
 \centering
 \includegraphics[width=0.9\linewidth,scale=0.5]{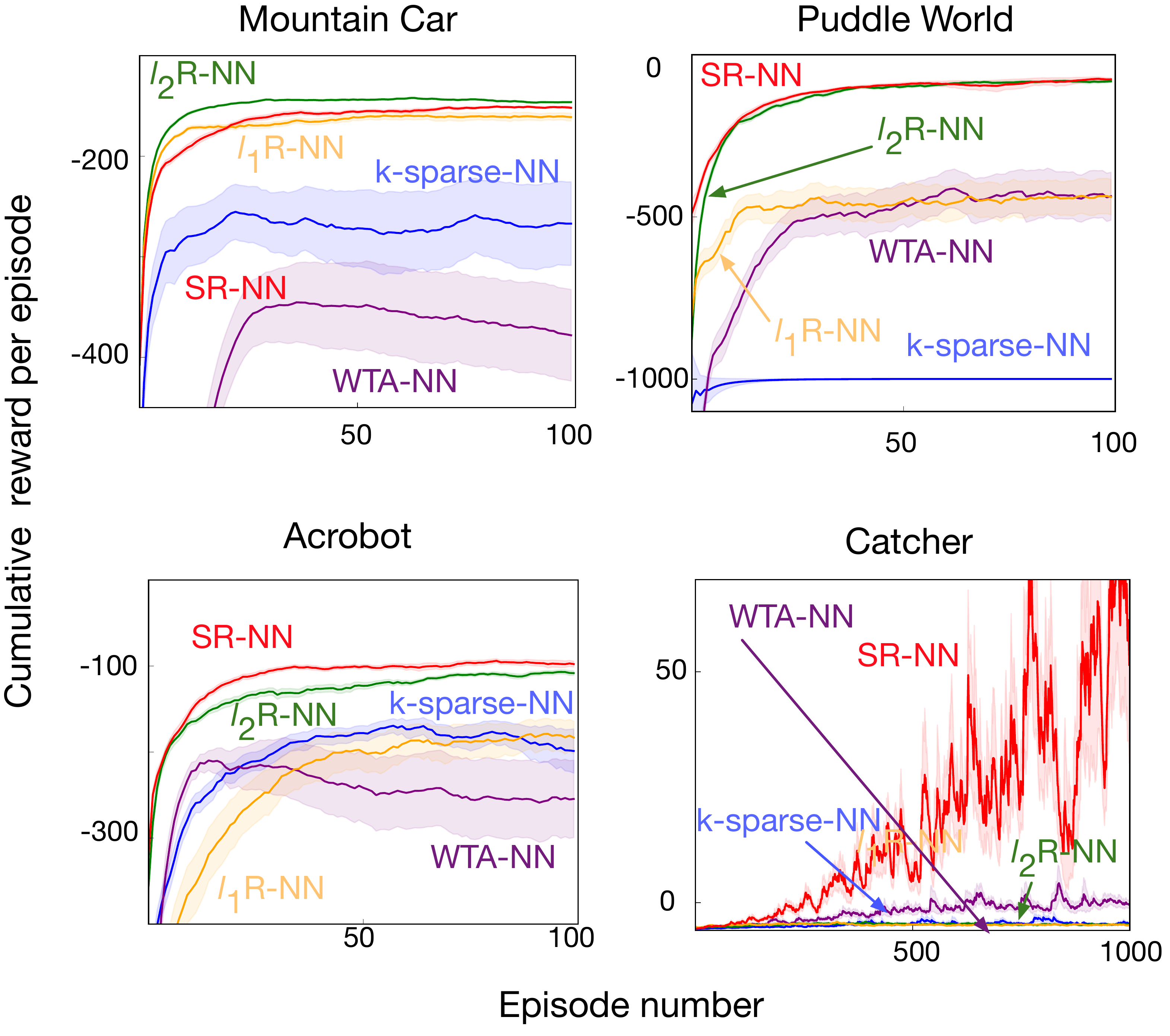}
 \caption{Learning curves for Sarsa(0) comparing SR-NN to previous proposed sparse representations learning strategies.}
 \label{fg:exp2-k-sparse}
\end{figure}

\begin{figure}[ht]
 \centering
 \includegraphics[width=0.9\linewidth]{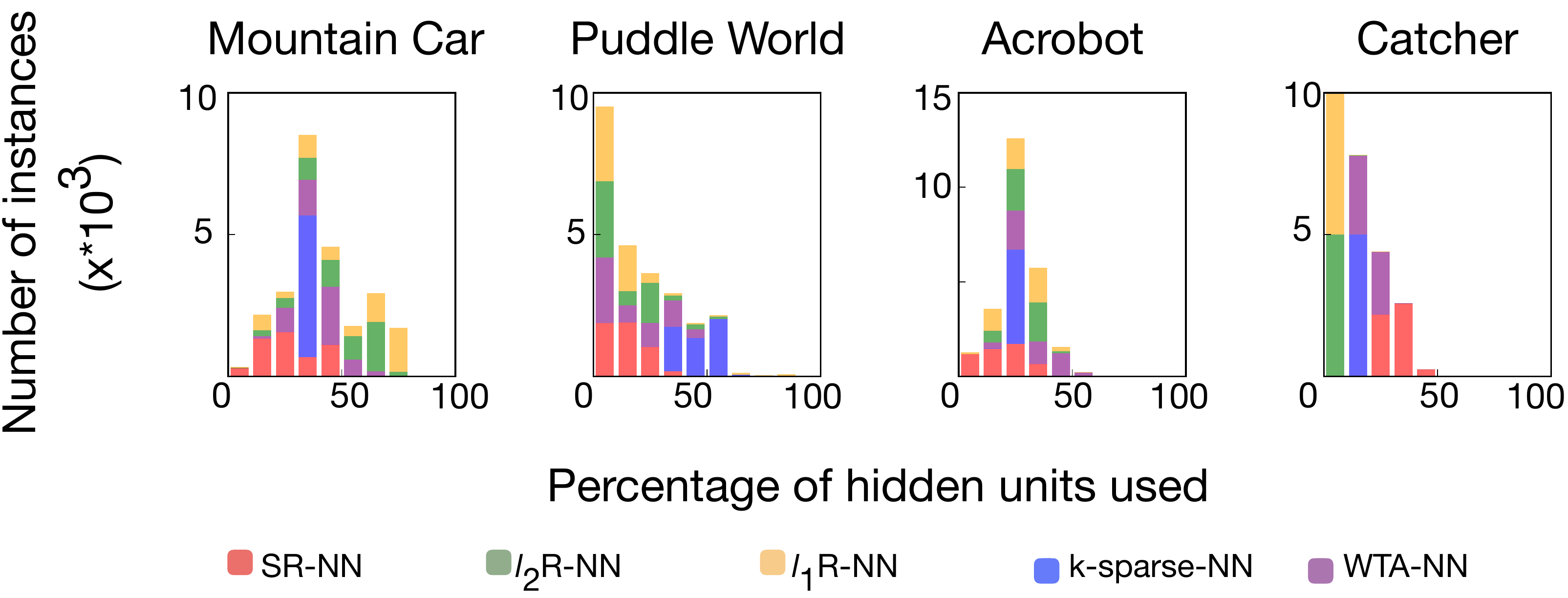}
 \caption{Instance sparsity comparing SR-NN to previous proposed sparse representations learning strategies.
 }
 \label{fg:hist-sparse}
\end{figure}

Next, we compare to previously proposed strategies for learning sparse representations with neural networks.
These include using $\ell_1$ and $\ell_2$ regularization on the activation (denoted by $\ell_1$R-NN and $\ell_2$R-NN respectively); k-sparse NNs, where all but the top $k$ activations are zeroed \cite{makhzani2013k} (k-sparse-NN); and Winner-Take-All NNs that keep the top k\% of the activations per node across instances, to promote sparse activations of nodes over time \cite{makhzani2015winner} (WTA-NN).\footnote{Both k-sparse-NNs and WTA-NNs were introduced for auto-encoders, though the idea can be applied more generally to NNs. We additionally tested these methods with autoencoders, but performance was significantly worse.}

We include learning curves and instance sparsity for these methods, for a ReLU activation, in Figures \ref{fg:exp2-k-sparse} and \ref{fg:hist-sparse}. Results for the Sigmoid activation are included in the Appendix. Neither WTA-NN nor k-sparse-NN are effective. We found the k-sparse-NN was prone to dead units, and often truncates non-negligible value.
Surprisingly, $\ell_2$R-NN performs comparably to SR-NN in all domains but Catcher, whereas $\ell_1$R-NN is effective only during early learning in Mountain Car. From the instance sparsity plots in Catcher, we see that $\ell_1$R-NN and $\ell_2$R-NN produce highly sparse (2\%-3\% instance sparsity), potentially explaining its poor performance. While similar instance sparsity was effective in Puddle World, this is unlikely to be true in general. This was with considerable parameter optimization for the regularization parameter.

\section{Conclusion}

In this work, we investigate using and learning sparse representations with neural networks for control in reinforcement learning. We show that sparse representations can significantly improve control performance when used in an incremental learning setting, and provide some evidence that this is because the locality of the representation reduces catastrophic interference which would otherwise overwrite bootstrap values.
We formalize Distributional Regularizers, with a practically important extension to a Set KL, for learning
a Sparse Representation Neural Network (SR-NN).
We provide an empirical investigation into the sparsity properties and control performance under different Distributional Regularizers, as well as compared to other algorithms to obtain sparse representations with neural networks. We conclude that SR-NN performs consistently well across domains, with the next best method---which only fails in one domain---being a simple methods that uses an $\ell_2$ regularizer on activations.

This work highlights an important phenomenon that arises in control, beyond the typical issues with catastrophic interference. Interference is typically considered for sequential multi-task learning, where previous functions are forgotten by training on a new task. Interference could occur even in a single-task setting, if an agent remains in a particular area of the space for a long time. In reinforcement learning, however, this problem is magnified by the fact that the agent uses its own estimates as targets. If estimates change incorrectly due to interference, there could be a cascading effect. This work provides some first empirical steps, in a carefully controlled set of experiments, to identify that this could be an issue, and that sparse representations could be a promising direction to alleviate the problem. We hope for this work to spur further empirical investigation into how widespread this issue is, and further algorithmic development into learning sparse representations for reinforcement learning.

\bibliography{paper.bib}

\begin{thebibliography}{}

\bibitem[\protect\citeauthoryear{Ahmad and Hawkins}{2015}]{ahmad2015properties}
Ahmad, S., and Hawkins, J.
\newblock 2015.
\newblock {Properties of Sparse Distributed Representations and their
  Application to Hierarchical Temporal Memory}.

\bibitem[\protect\citeauthoryear{Arpit \bgroup et al\mbox.\egroup
  }{2015}]{arpit2015regularized}
Arpit, D.; Zhou, Y.; Ngo, H.; and Govindaraju, V.
\newblock 2015.
\newblock Why regularized auto-encoders learn sparse representation?
\newblock In {\em International Conference on Machine Learning}.

\bibitem[\protect\citeauthoryear{Banerjee \bgroup et al\mbox.\egroup
  }{2005}]{banerjee2005clustering}
Banerjee, A.; Merugu, S.; Dhillon, I.~S.; and Ghosh, J.
\newblock 2005.
\newblock Clustering with bregman divergences.
\newblock {\em Journal of Machine Learning Research} 6(Oct):1705--1749.

\bibitem[\protect\citeauthoryear{Banino \bgroup et al\mbox.\egroup
  }{2018}]{banino2018vector}
Banino, A.; Barry, C.; Uria, B.; Blundell, C.; Lillicrap, T.; Mirowski, P.;
  Pritzel, A.; Chadwick, M.~J.; Degris, T.; Modayil, J.; et~al.
\newblock 2018.
\newblock Vector-based navigation using grid-like representations in artificial
  agents.
\newblock {\em Nature}.

\bibitem[\protect\citeauthoryear{Bengio, Courville, and
  Vincent}{2013}]{bengio2013representation}
Bengio, Y.; Courville, A.; and Vincent, P.
\newblock 2013.
\newblock Representation learning: A review and new perspectives.
\newblock {\em IEEE Transactions on Pattern Analysis and Machine Intelligence}.

\bibitem[\protect\citeauthoryear{Bengio}{2009}]{bengio2009learning}
Bengio, Y.
\newblock 2009.
\newblock {Learning deep architectures for AI}.
\newblock {\em Foundations and Trends in Machine Learning}.

\bibitem[\protect\citeauthoryear{Cover}{1965}]{cover1965geometrical}
Cover, T.~M.
\newblock 1965.
\newblock {Geometrical and Statistical Properties of Systems of Linear
  Inequalities with Applications in Pattern Recognition.}
\newblock {\em IEEE Trans. Electronic Computers}.

\bibitem[\protect\citeauthoryear{F{\"o}ldi{\'a}k}{1990}]{foldiak1990forming}
F{\"o}ldi{\'a}k, P.
\newblock 1990.
\newblock {Forming sparse representations by local anti-Hebbian learning}.
\newblock {\em Biological Cybernetics}.

\bibitem[\protect\citeauthoryear{French}{1991}]{french1991using}
French, R.~M.
\newblock 1991.
\newblock Using semi-distributed representations to overcome catastrophic
  forgetting in connectionist networks.
\newblock In {\em Annual Cognitive Science Society Conference}.

\bibitem[\protect\citeauthoryear{Glorot, Bordes, and
  Bengio}{2011}]{glorot2011deep}
Glorot, X.; Bordes, A.; and Bengio, Y.
\newblock 2011.
\newblock {Deep sparse rectifier neural networks}.
\newblock In {\em International Conference on Artificial Intelligence and
  Statistics}.

\bibitem[\protect\citeauthoryear{Goodfellow \bgroup et al\mbox.\egroup
  }{2009}]{goodfellow2009measuring}
Goodfellow, I.; Lee, H.; Le, Q.~V.; Saxe, A.; and Ng, A.~Y.
\newblock 2009.
\newblock {Measuring invariances in deep networks}.
\newblock In {\em Advances in Neural Information Processing Systems}.

\bibitem[\protect\citeauthoryear{He \bgroup et al\mbox.\egroup
  }{2015}]{he2015delving}
He, K.; Zhang, X.; Ren, S.; and Sun, J.
\newblock 2015.
\newblock Delving deep into rectifiers: Surpassing human-level performance on
  imagenet classification.
\newblock In {\em IEEE International Conference on Computer Vision}.

\bibitem[\protect\citeauthoryear{Hinton, Srivastava, and
  Swersky}{2012}]{hinton2012neural}
Hinton, G.; Srivastava, N.; and Swersky, K.
\newblock 2012.
\newblock Neural networks for machine learning lecture 6a overview of
  mini-batch gradient descent.

\bibitem[\protect\citeauthoryear{Kanerva}{1988}]{kanerva1988sparse}
Kanerva, P.
\newblock 1988.
\newblock {\em {Sparse Distributed Memory}}.
\newblock MIT Press.

\bibitem[\protect\citeauthoryear{Kingma and Ba}{2014}]{kingma2014adam}
Kingma, D.~P., and Ba, J.
\newblock 2014.
\newblock Adam: A method for stochastic optimization.
\newblock {\em arXiv:1412.6980}.

\bibitem[\protect\citeauthoryear{Le, Kumaraswamy, and
  White}{2017}]{le2017learning}
Le, L.; Kumaraswamy, R.; and White, M.
\newblock 2017.
\newblock {Learning sparse representations in reinforcement learning with
  sparse coding}.
\newblock {\em arXiv:1707.08316}.

\bibitem[\protect\citeauthoryear{Lee \bgroup et al\mbox.\egroup
  }{2008}]{lee2008sparse}
Lee, H.; Ekanadham, C.; information, A. N. A. i.~n.; and {2008}.
\newblock 2008.
\newblock {Sparse deep belief net model for visual area V2}.
\newblock In {\em Advances in Neural Information Processing Systems}.

\bibitem[\protect\citeauthoryear{Lemme, Reinhart, and
  Steil}{2012}]{lemme2012online}
Lemme, A.; Reinhart, R.~F.; and Steil, J.~J.
\newblock 2012.
\newblock {Online learning and generalization of parts-based image
  representations by non-negative sparse autoencoders}.
\newblock {\em Neural Networks}.

\bibitem[\protect\citeauthoryear{Mairal \bgroup et al\mbox.\egroup
  }{2009}]{mairal2009supervised}
Mairal, J.; Bach, F.; Ponce, J.; Sapiro, G.; and Zisserman, A.
\newblock 2009.
\newblock {Supervised dictionary learning}.
\newblock In {\em Advances in Neural Information Processing Systems}.

\bibitem[\protect\citeauthoryear{Mairal \bgroup et al\mbox.\egroup
  }{2010}]{mairal2010online}
Mairal, J.; Bach, F.; Ponce, J.; and Sapiro, G.
\newblock 2010.
\newblock {Online learning for matrix factorization and sparse coding}.
\newblock {\em Journal of Machine Learning Research}.

\bibitem[\protect\citeauthoryear{Makhzani and Frey}{2013}]{makhzani2013k}
Makhzani, A., and Frey, B.
\newblock 2013.
\newblock {k-sparse autoencoders}.
\newblock {\em arXiv preprint arXiv:1312.5663}.

\bibitem[\protect\citeauthoryear{Makhzani and Frey}{2015}]{makhzani2015winner}
Makhzani, A., and Frey, B.
\newblock 2015.
\newblock {Winner-take-all autoencoders}.
\newblock In {\em Adv. in Neural Information Processing Systems}.

\bibitem[\protect\citeauthoryear{McCloskey and
  Cohen}{1989}]{mccloskey1989catastrophic}
McCloskey, M., and Cohen, N.~J.
\newblock 1989.
\newblock {Catastrophic Interference in Connectionist Networks: The Sequential
  Learning Problem}.
\newblock {\em Psychology of Learning and Motivation}.

\bibitem[\protect\citeauthoryear{McCulloch and
  Pitts}{1943}]{mcculloch1943logical}
McCulloch, W.~S., and Pitts, W.
\newblock 1943.
\newblock {A logical calculus of the ideas immanent in nervous activity}.
\newblock {\em The Bulletin of Mathematical Biophysics}.

\bibitem[\protect\citeauthoryear{Ng}{2011}]{ng2011sparse}
Ng, A.
\newblock 2011.
\newblock {Sparse autoencoder}.
\newblock {\em CS294A Lecture notes}.

\bibitem[\protect\citeauthoryear{Olshausen and
  Field}{1997}]{olshausen1997sparse}
Olshausen, B.~A., and Field, D.~J.
\newblock 1997.
\newblock {Sparse coding with an overcomplete basis set: A strategy employed by
  V1?}
\newblock {\em Vision Research}.

\bibitem[\protect\citeauthoryear{Olshausen}{2002}]{olshausen2002sparse}
Olshausen, B.~A.
\newblock 2002.
\newblock {Sparse Codes and Spikes}.
\newblock In {\em Probabilistic Models of the Brain}.

\bibitem[\protect\citeauthoryear{Quian~Quiroga and
  Kreiman}{2010}]{quian2010measuring}
Quian~Quiroga, R., and Kreiman, G.
\newblock 2010.
\newblock Measuring sparseness in the brain: Comment on bowers (2009).

\bibitem[\protect\citeauthoryear{Ranzato \bgroup et al\mbox.\egroup
  }{2006}]{ranzato2006efficient}
Ranzato, M.; Poultney, C.~S.; Chopra, S.; and LeCun, Y.
\newblock 2006.
\newblock {Efficient Learning of Sparse Representations with an Energy-Based
  Model.}
\newblock In {\em Adv. in Neural Info. Process. Sys.}

\bibitem[\protect\citeauthoryear{Ranzato, Boureau, and
  LeCun}{2007}]{ranzato2007sparse}
Ranzato, M.; Boureau, Y.-L.; and LeCun, Y.
\newblock 2007.
\newblock {Sparse Feature Learning for Deep Belief Networks.}
\newblock In {\em Advances in Neural Information Processing Systems}.

\bibitem[\protect\citeauthoryear{Ratitch and Precup}{2004}]{ratitch2004sparse}
Ratitch, B., and Precup, D.
\newblock 2004.
\newblock {Sparse distributed memories for on-line value-based reinforcement
  learning}.
\newblock In {\em Machine Learning: ECML PKDD}.

\bibitem[\protect\citeauthoryear{Riedmiller}{2005}]{riedmiller2005neural}
Riedmiller, M.
\newblock 2005.
\newblock {Neural fitted Q iteration--first experiences with a data efficient
  neural reinforcement learning method}.
\newblock In {\em European Conference on Machine Learning}.

\bibitem[\protect\citeauthoryear{Rifai \bgroup et al\mbox.\egroup
  }{2011}]{rifai2011contractive}
Rifai, S.; Vincent, P.; Muller, X.; Glorot, X.; and Bengio, Y.
\newblock 2011.
\newblock {Contractive auto-encoders: Explicit invariance during feature
  extraction}.
\newblock In {\em Inter. Conf. on Machine Learning}.

\bibitem[\protect\citeauthoryear{Srivastava \bgroup et al\mbox.\egroup
  }{2014}]{srivastava2014dropout}
Srivastava, N.; Hinton, G.; Krizhevsky, A.; Sutskever, I.; and Salakhutdinov,
  R.
\newblock 2014.
\newblock Dropout: a simple way to prevent neural networks from overfitting.
\newblock {\em The Journal of Machine Learning Research} 15(1):1929--1958.

\bibitem[\protect\citeauthoryear{Sutton and
  Barto}{1998}]{sutton1998reinforcement}
Sutton, R.~S., and Barto, A.~G.
\newblock 1998.
\newblock {\em {Reinforcement learning: An introduction}}.
\newblock MIT press Cambridge.

\bibitem[\protect\citeauthoryear{Sutton}{1988}]{sutton1988learning}
Sutton, R.~S.
\newblock 1988.
\newblock Learning to predict by the methods of temporal differences.
\newblock {\em Machine learning} 3(1):9--44.

\bibitem[\protect\citeauthoryear{Sutton}{1996}]{sutton1996generalization}
Sutton, R.~S.
\newblock 1996.
\newblock {Generalization in reinforcement learning: Successful examples using
  sparse coarse coding}.
\newblock In {\em Advances in Neural Information Processing Systems}.

\bibitem[\protect\citeauthoryear{Teh \bgroup et al\mbox.\egroup
  }{2003}]{teh2003energy}
Teh, Y.~W.; Welling, M.; Osindero, S.; and Hinton, G.~E.
\newblock 2003.
\newblock {Energy-Based Models for Sparse Overcomplete Representations.}
\newblock {\em Journal of Machine Learning Research}.

\bibitem[\protect\citeauthoryear{Triesch}{2005}]{triesch2005agradient}
Triesch, J.
\newblock 2005.
\newblock {A Gradient Rule for the Plasticity of a Neuron{\textquoteright}s
  Intrinsic Excitability}.
\newblock In {\em ICANN}.

\bibitem[\protect\citeauthoryear{White}{2017}]{white2017unifying}
White, M.
\newblock 2017.
\newblock {Unifying Task Specification in Reinforcement Learning}.
\newblock In {\em Inter. Conf. on Machine Learning}.

\end{thebibliography}
{\small
\bibliographystyle{aaai}
}

\newpage
\appendix
\section{Additional algorithmic details}

\begin{algorithm}
\caption{Optimizing the regularized objective}
\label{alg:opt}
\hspace*{\algorithmicindent}
\begin{algorithmic}[1]
    \State Initialize neural networks weights based on He initialization \cite{he2015delving}: for each layer $l$ and each element $ij$ of the weight matrix  $\Wmat^{(l)}_{ij}\sim\calN(0, \frac{2}{n_l})$ and $\bvec^{(l)} = \boldsymbol{0}$  where $n_l$ the number of input nodes for layer $l$.
    \While{not converge to a minimum}
        \State Draw $m$ i.i.d. samples $\{y_1, ..., y_m\}$ from the true data distribution
        \State For $j=1,...,k$, compute $\betahat_j=\sum_{i=1}^{m} y_{ij} / m$ and the gradient:
        \[\frac{\partial KL(p_{\beta}||p_{\betahat_j})}{\partial \betahat_j}  = (\frac{1}{\betahat_j} - \frac{\beta}{\betahat_j^2})\mathbbm{1}[\betahat_j>\beta]\]
        \State Update each weight $\thetavec\in\{\forall l, \Wmat^{(l)}, \bvec^{(l)}\}$ with the gradient:
        \[\frac{\partial J(\thetavec)}{\partial \thetavec} + \lambda_{KL} \sum_{j=1}^{k} \frac{\partial KL(p_{\beta}||p_{\betahat_j})}{\partial \betahat_j}  \frac{\partial \betahat_j}{\partial \thetavec}\]
    \EndWhile
\end{algorithmic}
\end{algorithm}

In general, we advocate for learning the representation incrementally, for the task faced by the agent. However, for our experiments, we learned the representations first to remove confounding factors. We detail that learning regime here.

The problem of learning a good representation $\phivec_\thetavec(s,a)$ in the case of finite actions can be transformed to learning a good representation of the form $\phivec_\thetavec(s)$, and using that to represent the action-value function from Equation (\ref{eq:q_sa_1}) as:
\begin{equation}
\hat{Q}_{\weights, \thetavec}(s, a) \defeq \phivec_\thetavec(s)^\top \weights_a
\label{eq:q_sa_2}
\end{equation}
Here, $\phivec_\thetavec(s)$ is the linear representation of the state $s$, which is used in conjunction with the linear predictor $\weights_a$ to estimate action-values for action $a$ across the state space. Under a given policy, like the action-values $Q^\pi(s,a)$, corresponding state-values, $V^\pi(s)$, are defined as:
\begin{align*}
    V^\pi(s) &\defeq \mathbb{E}[G_t | S_t = s] \\
    &\text{ where, } G_t = R_{t+1} + \gamma_{t+1} G_{t+1}
\end{align*}
An easy objective to train connectionist networks with simple backpropagation is the Mean Squared Temporal Difference Error (MSTDE) \cite{sutton1988learning}. For a given policy, the MSTDE is defined as:
\begin{align}
  \label{eq:mstde}
  \sum_{s \in \States} &\mathbf{d}(s) \mathbb{E}[\delta_t^2 | S_t=s]\\
  \text{where, } \delta_t \defeq R_{t+1} &+ \gamma_{t+1} \phivec_\thetavec(S_{t+1})^\top \weights_v - \phivec_\thetavec(S_t)^\top \weights_v \nonumber
\end{align}
Here, $\mathbf{d}$ denotes the stationary distribution over the states induced by the given policy, and $\thetavec$ and $\weights_v$ are parameters that can be estimated with stochastic gradient descent. Therefore, given experience generated by a policy that explores sufficiently in an environment, a strong function approximator (a dense neural network) can be trained to estimate useful features, $\phivec_\thetavec(s)$. These features can then be used for estimating action-values in on-policy control for learning the (close-to) optimal behaviour policy in the environment.

\begin{figure}
  \centering
  \includegraphics[width=\linewidth]{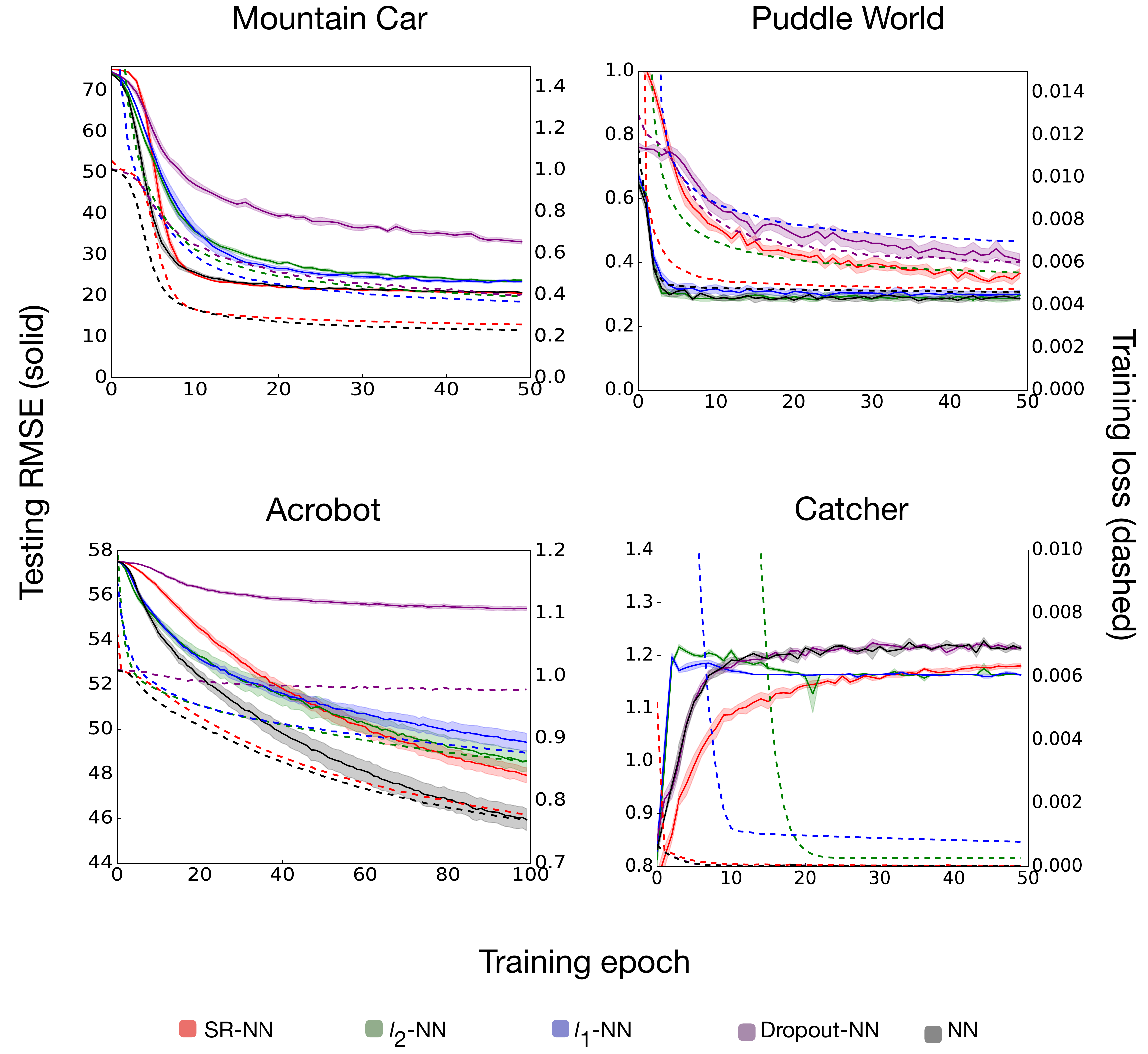}
  \caption{Learning curve during the training of neural networks with different regularization methods.}
  \label{fg:curve-reg}
\end{figure}

\section{Experimental Details}
\label{exp_detail}
\subsection{Policies to generate training and testing data}
In Mountain Car, we use the standard energy pumping policy with 10\% randomness. In Puddle World, by a policy that chooses to go North with 50$\%$ probability, and East with 50$\%$ probability on each step.
The data in Acrobot is generated by a near-optimal policy. In Catcher, the agent chooses to move toward the apple with 50$\%$ probability, and selects a random action with 50$\%$ probability on each step; and gets only 1 life in the environment.

\subsection{Tile Coding}
We compare to Tile Coding (TC) representation, a well-known sparse representation, as the baseline. TC uses overlapping grids on the observation space, to convert a continuous space to a discrete dimensional space. The representations generated by it are sparse and distributed based on a static hashing technique. We experiment with several configurations for the fixed representation, particularly with grid-sizes(N) in $\{4,8,16\}$ and number of tilings (D) in $\{8,16,32\}$.
We use a hash size of 8192, which is significantly larger than the largest feature size of 256, as used in the other learned representation models we compare to. The results shown in Figure \ref{fg:control-reg} are for the best configuration of the static tile-coder after a sweep.

\subsection{Training neural networks}
\textbf{Architecture and optimizer}: We used neural networks with two hidden layers. The first layer 32 hidden units. The second layer, which is the representation layer used for prediction, has 256 units. We optimized the neural network weights using Adam optimization \cite{kingma2014adam} with a batch size of 64. The neural network weights are initialized based on He initialization \cite{he2015delving}. That is, the neural networks weights are initialized with zero-mean Gaussian distribution with variance equals to $2/n_l$, where $n_l$ is the number of input nodes for layer $l$.

\begin{figure}
  \centering
  \includegraphics[width=\linewidth]{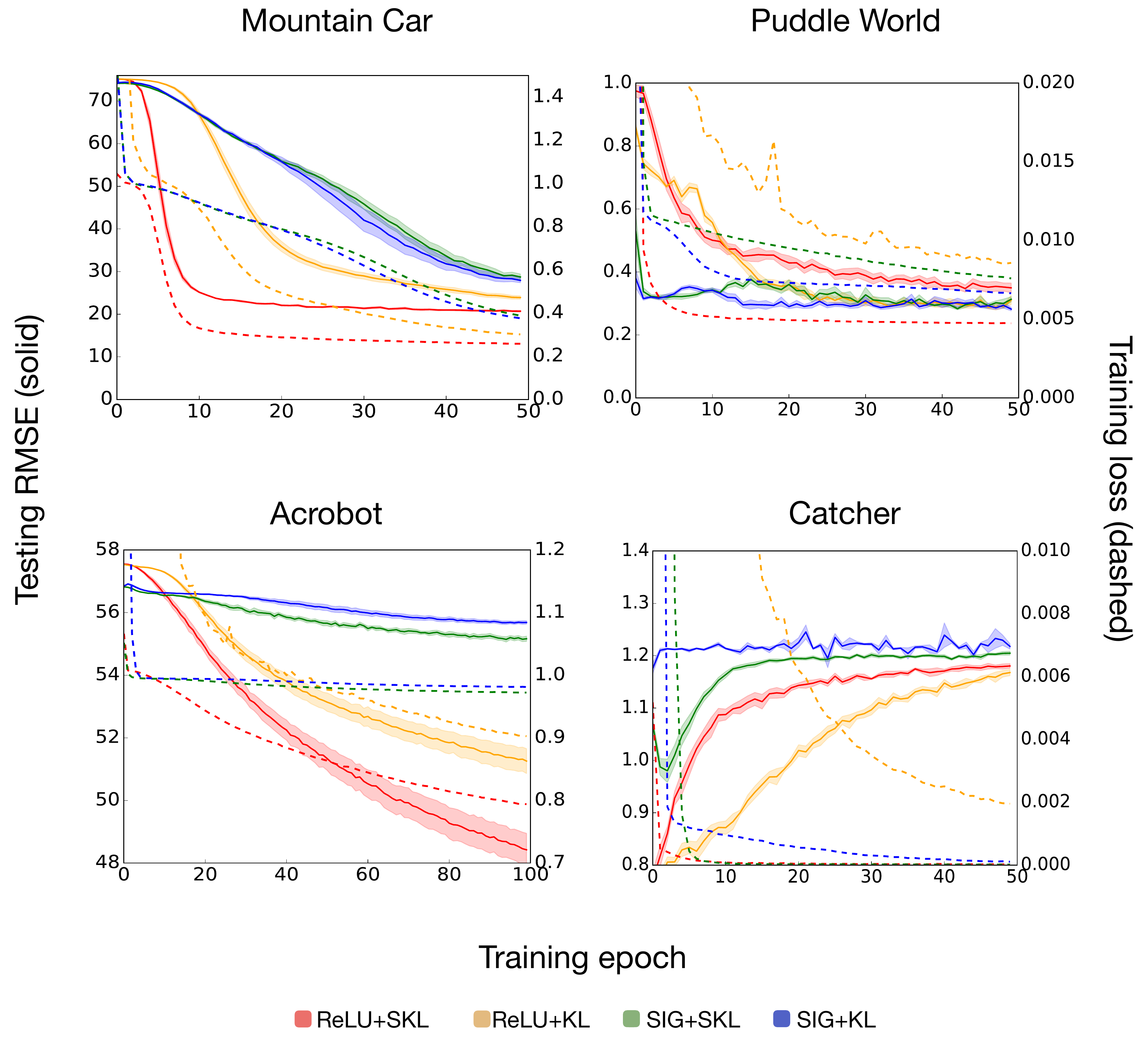}
  \caption{Learning curve during the training of neural networks with different distributional regularizers.}
  \label{fg:curve-dist}
\end{figure}

\textbf{Representation hyperparameters}: The range of grid search for the representation hyperparameters are as follows:
\begin{align*}
    \lambda_{KL} & \in \{0.1, 0.01, 0.001\} \\
    \beta & \in \{0.05, 0.1, 0.2\} \\
    \lambda_{NN} \text{ for $\ell_1$ and $\ell_2$}  & \in \{0.1, 0.01, 0.001, 0.0001\} \\
    \text{dropout probability } p  & \in \{0.1, 0.2, 0.3, 0.4, 0.5\} \\
    k \text{ for k-sparse} & \in \{16, 32, 64, 128\} \\
    k \text{ for WTA} & \in \{6.25\%, 12.5\%, 25\%, 50\%\} \\
\end{align*}

\textbf{Algorithmic choices}: For k-sparse networks, only the top-k hidden units in the representation layer are activated. We also use scheduling of sparsity level described in the original paper \cite{makhzani2013k}. If used in conjunction with a distributional regularizer, the top-k nodes are chosen before application of the distributional regularizer. For dropout, given the form of the supervision goal (MSTDE), the same dropout mask is chosen to generate the representation for both states $S_{t+1}$ and $S_t$\footnote{We have experimented with different dropout masks for $S_{t+1}$ and $S_t$, and the result suggests that it is not able to learn good representations even for prediction across all domains.} -- this preserves dropouts role as regularizer w.r.t. the target, and promotes diversity in learning.

\textbf{Grid-search evaluation metric}: The learned representations are then used for on-policy control in Sarsa(0) with fixed $\epsilon=0.1$. The value function for Sarsa is initialized with zero-mean Gaussian distribution with small variance.
For sparse representations, we use semi-gradient Sarsa with step decay learnining rate. For dense representatinos, we use adaptive learnining rate method RMSprop \cite{hinton2012neural}. The initial learning rate for Sarsa(0) is swept in the set:
\begin{align*}
    \alpha_0 & \in \{0.1, 0.04, 0.01, 0.004, 0.001, 0.0004, 0.0001\}
\end{align*}
All the sweeps for selecting the representation learning hyperparameters across domains use 50 epochs and 10 runs.
%search to  used in Section \ref{exp} are trained on 10,000 transitions over a fixed number of epochs.

\begin{figure}
  \centering
  \includegraphics[width=\linewidth]{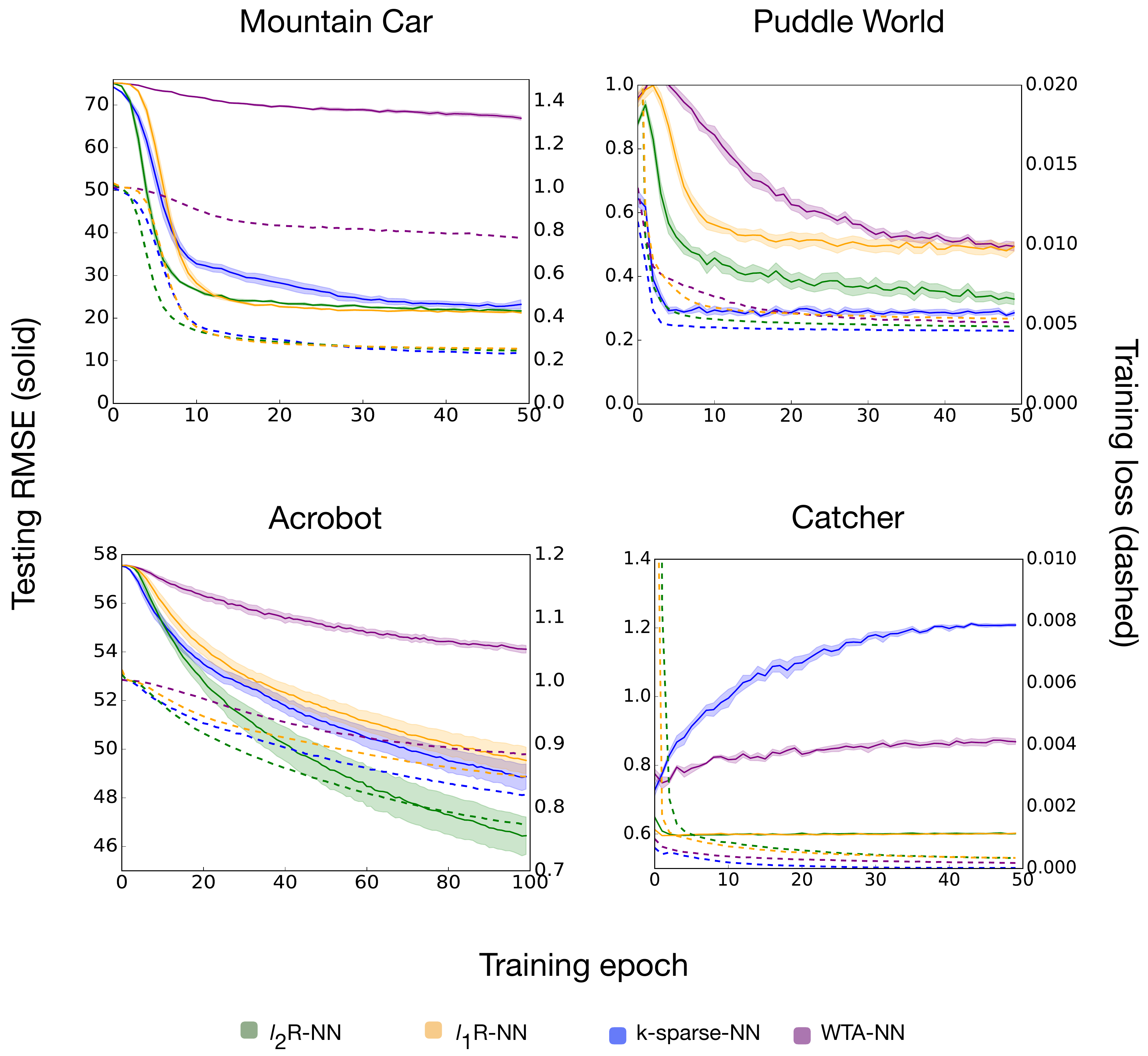}
  \caption{Learning curve during the training of neural networks with different sparsity inducing approaches.}
  \label{fg:curve-sparse}
\end{figure}

\begin{figure}
  \centering
  \includegraphics[width=0.98\linewidth]{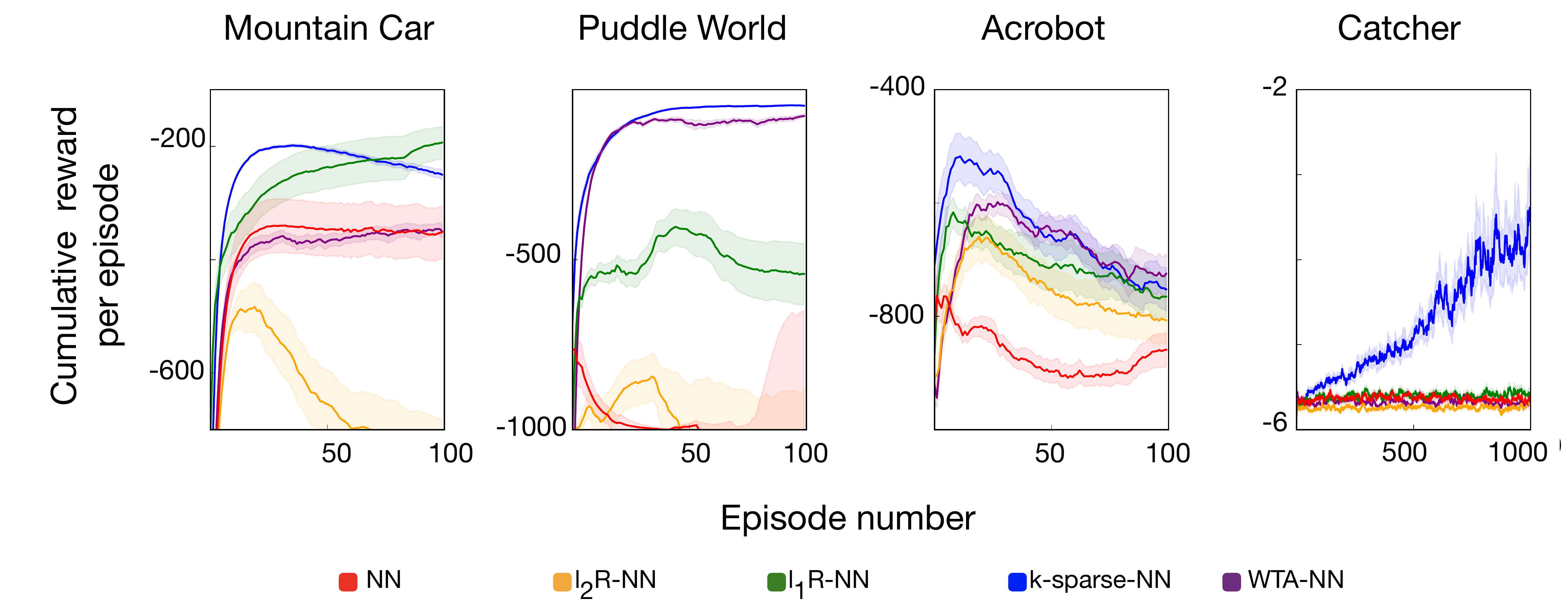}
  \caption{Learning curves for Sarsa(0) comparing various sparsity inducing networks with Sigmoid activation.}
  \label{fg:exp2-sigmoid}
\end{figure}

\begin{figure}
  \centering
  \includegraphics[width=0.98\linewidth]{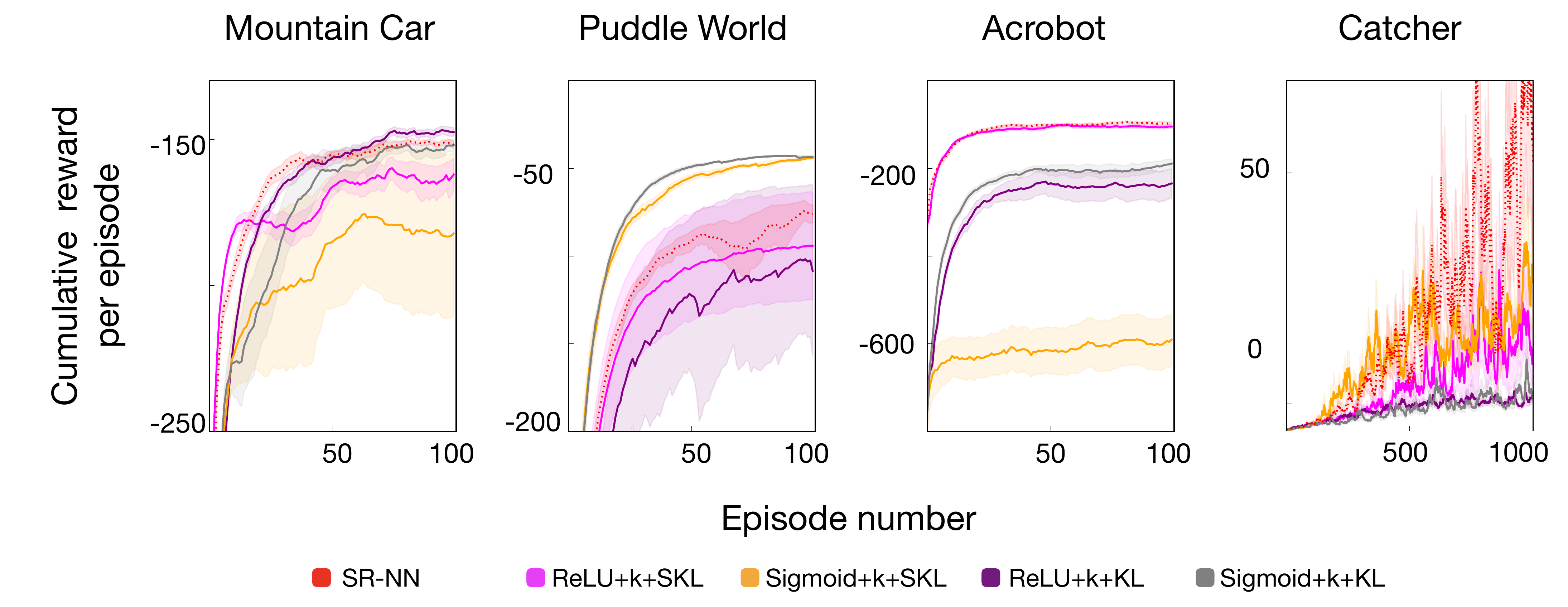}
  \caption{Learning curves for Sarsa(0) comparing various variants of distributional regularizers with k-sparse.}
  \label{fg:exp2-k-sparse2}
\end{figure}

\begin{figure}
  \centering
  \includegraphics[width=\linewidth]{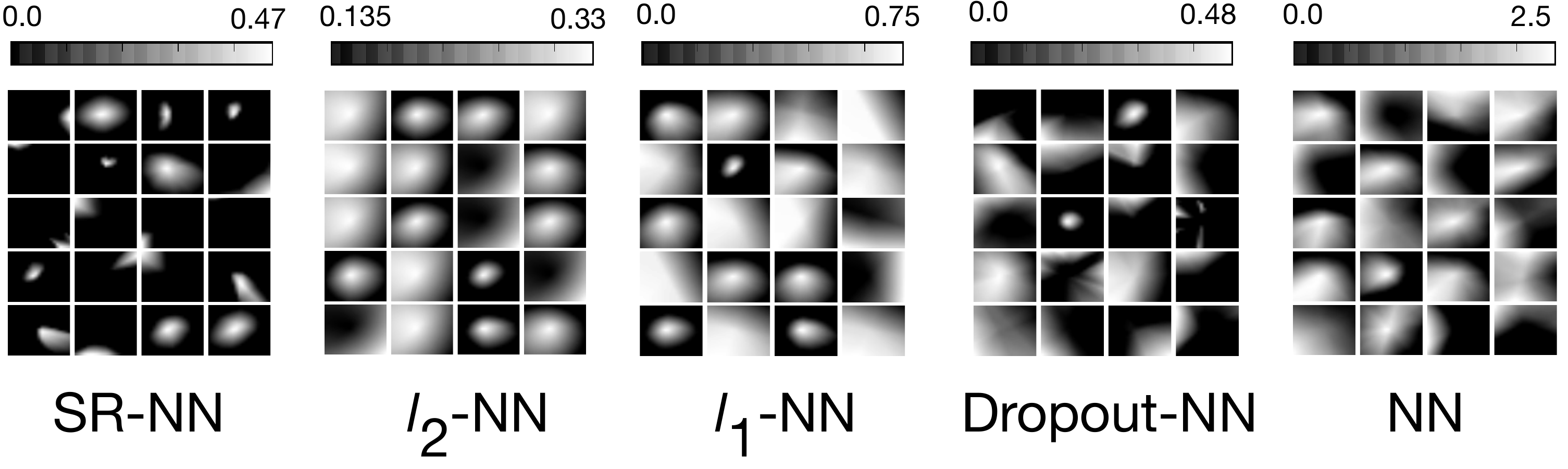}
  \caption{Heatmaps of activations comparing SR-NN to different regularization strategies and NN in Mountain Car.}
  \label{fg:heatmap-mc-reg}
\end{figure}

\begin{figure}
  \centering
  \includegraphics[width=\linewidth]{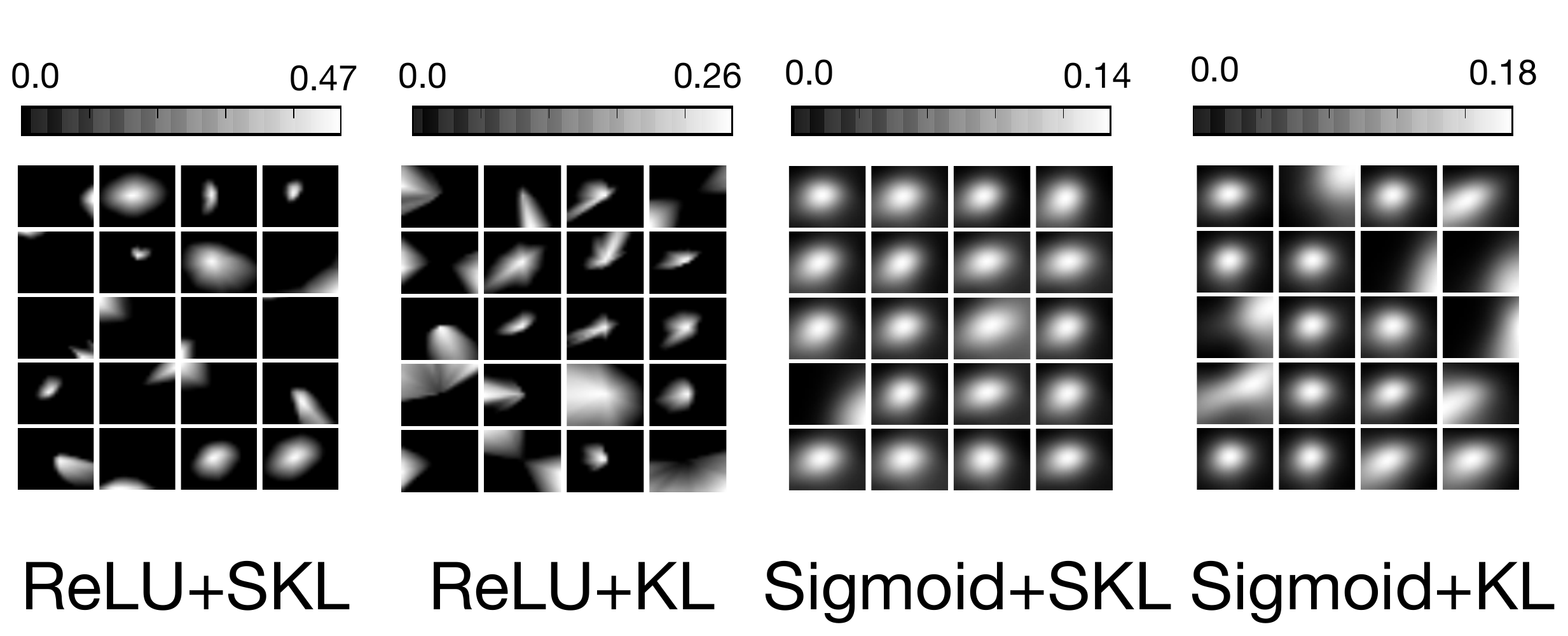}
  \caption{Heatmaps of activations with different Distributional Regularizers in Mountain Car.}
  \label{fg:heatmap-mc-dist}
\end{figure}

\textbf{Learning curves}: The chosen hyperparameters are used to train a good representation (saturated testing loss -- 100 epochs for Acrobot, and 50 epochs for other domains), following which it is used for on-policy control with Sarsa(0). While the control performance is focused on in the main paper, the learning curve during the representation training phase is shown in Figures \ref{fg:curve-reg}, \ref{fg:curve-dist} and \ref{fg:curve-sparse}. The metric on the y-axis is the Root Mean Squared Error (RMSE), which is evaluated as follows:
\begin{align*}
\text{RMSE} = \sqrt{\frac{\sum_{\xvec\in \Xmat_{test}}(\hat{V}(\xvec)-V^*(\xvec))^2}{\Xmat_{test}}}
\end{align*}
where $\Xmat_{test}$ is the set of test
states for which the representations have been extracted, $\hat{V}(\xvec)$ is the estimated value of state $\xvec$ and $V^*(\xvec)$ is the true value of state $\xvec$ computed using Monte Carlo rollouts. The number of test states are 5000 for benchmark domains and 1000 for Catcher.
Most algorithms converge to a good solution within 50 epochs in Mountain Car, Puddle World and Catcher, and 100 epochs in Acrobot as shown in the curves. The curves are for representation purposes and only averaged over 5 runs.
All learning curves for Sarsa(0) are averaged over 30 runs, and are plotted with exponential moving average ($\beta = 0.1$).

\section{More results}
\subsection{Control curves}
We perform the evaluation of sparsity inducing networks with Sigmoid activation. Figure \ref{fg:exp2-sigmoid} shows the performance of Sarsa(0) with representations learned by different networks. k-sparse and WTA performs well in Puddle World, however, none of these representations are effective across all domains.

The learning curves for various k-sparse networks with distributional regularizers are in Figure \ref{fg:exp2-k-sparse2}. It suggests that k-sparse (ReLU+k+SKL) provides no improvement over just using distributional regularizer for ReLU activation (SR-NN).

\begin{figure*}
  \centering
  \includegraphics[width=\linewidth]{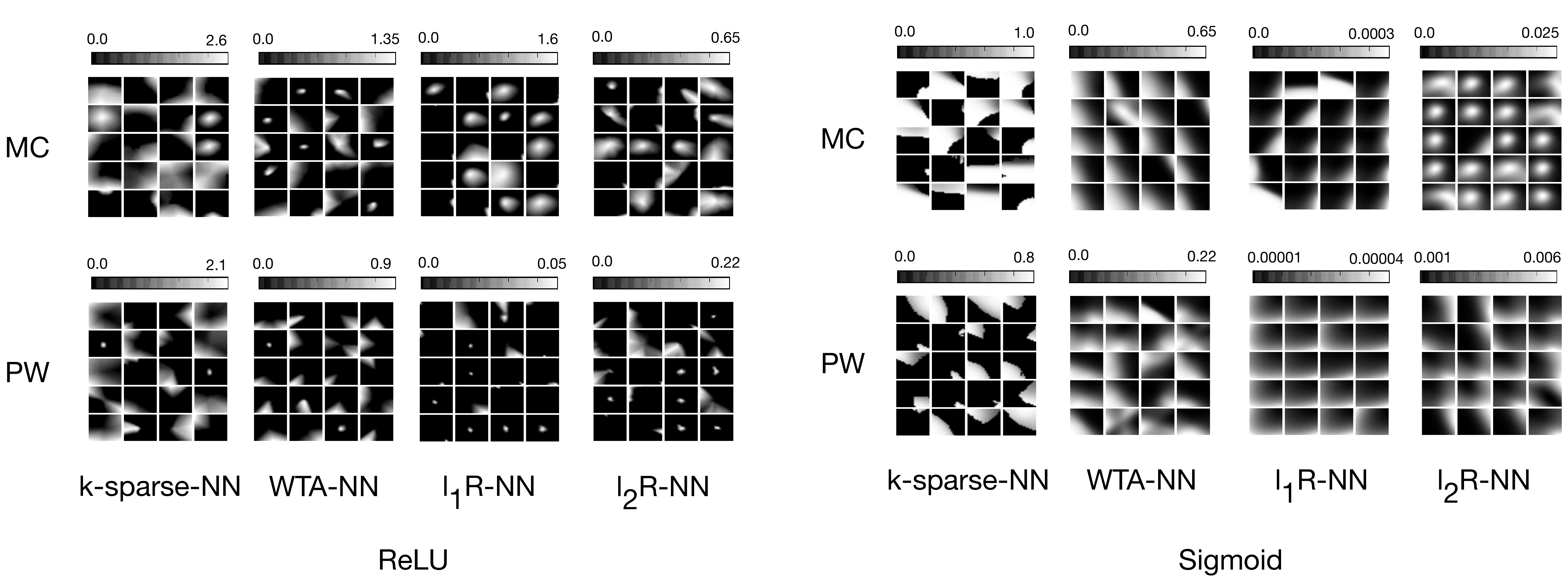}
  \caption{Heatmaps of activations for nodes from other networks which aim to generate sparse representations (ReLU and Sigmoid activation).}
  \label{fg:heatmap-sparse-both}
\end{figure*}

\subsection{Activation heatmaps}
The activation heatmaps for randomly selected neurons (excluding dead neurons) in Mountain Car with different regularization stratergies are shown in Figure \ref{fg:heatmap-mc-reg}, and with differnt Distributional Regularization designs are shown in Figure \ref{fg:heatmap-mc-dist}.
Heatmaps for sparsity inducing networks with ReLU activations and Sigmoid activation, for Mountain Car and Puddle World are shown in Figure \ref{fg:heatmap-sparse-both}.

\begin{figure*}
  \includegraphics[width=\linewidth]{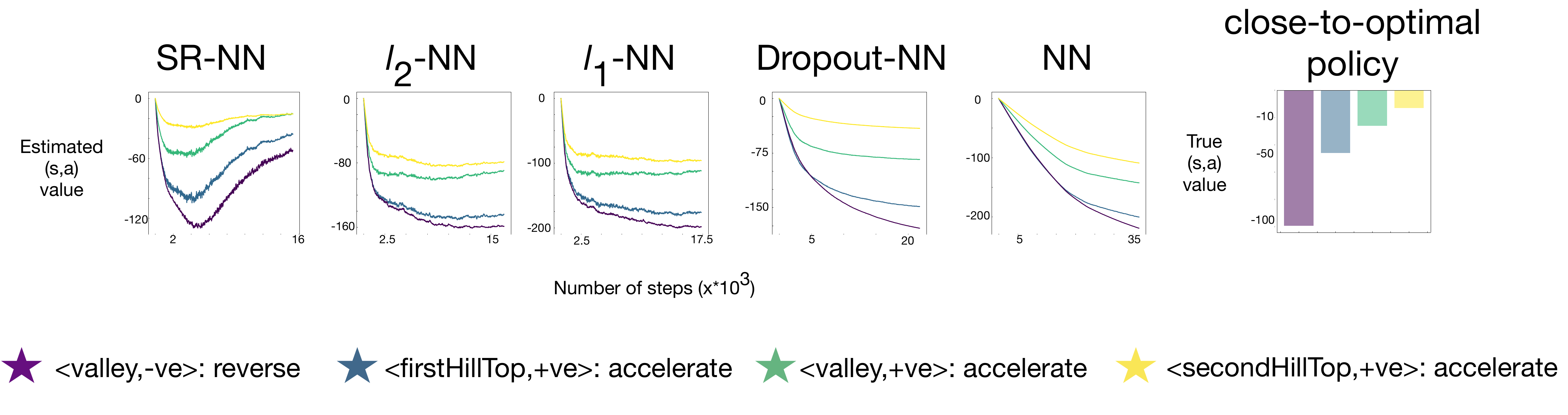}
  \caption{This plot compares the bootstrap estimates of SR-NN to various regularization stratergies during on-policy control for the chosen 4 state-action pairs denoted in the following format in the legend: <car-position,car-velocity>:action. Again, we see that the relative ordering of bootstrap values is maintained with SR-NN, and it tends towards the true values of the ($\epsilon=0.1$)-optimal policy. The optimal policy estimates (currently) use 10k Monte Carlo rollouts with a powerful close-to-optimal tile-coder policy.}
  \label{fg:bootstrap_dist}
\end{figure*}

\subsection{Bootstrap values}
The bootstrap values comparing SR-NN to different regularization strategies, and NN are shown in Figure \ref{fg:bootstrap_dist}. Since it is not easy to visualize 4-dimensional space, we only include the bootstrap value result of Mountain Car here.

\subsection{Activation overlap}
We show the overlap of representations learned by different networks in Table \ref{tab:overlap} for Mountain Car and Puddle World. $\ell_2$R-NN and $\ell_1$R-NN have low overlap values. However, the regularizers tend to push many neurons to be activated for a really small region to reduce penalty as shown in Figure \ref{fg:heatmap-sparse-both}. SR-NN, on the other hand, learns a more distributed representation.

\begin{table}
\begin{center}
\begin{tabular}{ c | c c }
  & Mountain Car & Puddle World \\
  \hline
  SR-NN & 16.8   & 8.8 \\
  \hline
  $\ell_2$-NN & 112.3  & 111.5 \\
  \hline
  $\ell_1$-NN & 109.5  & 142.5 \\
  \hline
  Dropout-NN & 72.5  & 31.2 \\
  \hline
  NN      & 106.5 & 54.0 \\
  \hline
  ReLU+KL  & 36.8  & 71.4 \\
  \hline
  SIG+SKL & 256.0 & 256.0 \\
  \hline
  SIG+KL  & 256.0 & 256.0 \\
  \hline
  k-sparse-NN & 36.6 & 61.8 \\
  \hline
  WTA-NN & 24.8 & 6.5 \\
  \hline
  $\ell_2$R-NN & 30.0 & 3.8 \\
  \hline
  $\ell_1$R-NN & 10.5 & 0.4 \\
\end{tabular}
\end{center}
  \caption{Activation overlap in Mountain Car and Puddle World. For Mountain Car, the numbers are the average overlap over all pairs of selected states defined in Figure \ref{fg:bootstrap_dist}.}
  \label{tab:overlap}
\end{table}

\end{document}